%% file: main.tex
\documentclass[sigconf]{acmart}
\setcopyright{none}
\renewcommand\footnotetextcopyrightpermission[1]{}
\AtBeginDocument{%
  \providecommand\BibTeX{{%
    \normalfont B\kern-0.5em{\scshape i\kern-0.25em b}\kern-0.8em\TeX}}}

\setcopyright{acmlicensed}
\copyrightyear{2024}
\acmYear{2024}
\acmDOI{XXXXXXX.XXXXXXX}

\acmConference[KDD '24]{Make sure to enter the correct
  conference title from your rights confirmation emai}{Aug 25--29,
  2024}{Barcelona, Spain}
\acmISBN{978-1-4503-XXXX-X/18/06}

\usepackage{microtype}
\usepackage{graphicx}
\usepackage{booktabs} 

\usepackage{titletoc}
\usepackage{amsmath}
\usepackage{amsthm}

\usepackage{subfig}
\usepackage{multirow}
\usepackage{array}
\usepackage{multicol}
\usepackage{tabularx}
\usepackage{hyperref}

\include{include/custom}
\include{include/custom_paper}

\usepackage{algorithm}
\usepackage{algpseudocode}
\usepackage{dsfont}

\usepackage{graphicx}
\usepackage{booktabs}

\usetikzlibrary{shapes,decorations,arrows,calc,arrows.meta,fit,positioning}
\tikzset{
	-Latex,auto,node distance =1 cm and 1 cm,semithick,
	state/.style ={ellipse, draw, minimum width = 0.7 cm},
	point/.style = {circle, draw, inner sep=0.04cm,fill,node contents={}},
	bidirected/.style={Latex-Latex,dashed},
	el/.style = {inner sep=2pt, align=left, sloped}
}

\makeatletter
\newcommand{\customlabel}[2]{%
   \protected@write \@auxout {}{\string \newlabel {#1}{{#2}{\thepage}{#2}{#1}{}} }%
   \hypertarget{#1}{}%
}

\allowdisplaybreaks

\begin{document}

\title{Conformal Counterfactual Inference under Hidden Confounding}

\author{Zonghao Chen}
\authornote{Work done during an internship at ByteDance Research}
\authornote{Equal contribution}
\affiliation{%
\institution{University College London}
  \country{UK}
}
\email{zonghao.chen.22@ucl.ac.uk}

\author{Ruocheng Guo}
\authornotemark[2] 
\affiliation{%
  \institution{ByteDance Research}
  \country{UK}
}
\email{ruocheng.guo@bytedance.com}

\author{Jean-Fran\c cois Ton}
\affiliation{%
  \institution{ByteDance Research}
  \country{UK}
}
\email{jeanfrancois@bytedance.com}

\author{Yang Liu}
\affiliation{%
  \institution{ByteDance Research}
  \country{USA}
}
\email{yang.liu01@bytedance.com}
\renewcommand{\shortauthors}{Chen et al.}

 
\input{icml_2024/0_abstract}



\keywords{}




\maketitle

\input{icml_2024/1_introduction}
\input{icml_2024/2_background}
\input{icml_2024/3_methodology}
\input{icml_2024/4_two_stage}

\input{icml_2024/5_experiments}

\input{icml_2024/6_related_work}
\input{icml_2024/7_conclusion}
\clearpage

\bibliographystyle{unsrt}

\bibliography{main}

\clearpage

\input{icml_2024/appendix/main}

\end{document}

%% file: include/custom.tex
\usepackage[utf8]{inputenc}
\usepackage{booktabs} 
\usepackage{hyperref}

\usepackage{amsthm}
\usepackage{bbm} 
\usepackage{bm}
\usepackage{verbatim}
\usepackage{float}
\usepackage{color,soul}
\usepackage{enumitem}
\usepackage{mathtools}
\usepackage{hhline}
\usepackage[title]{appendix}
\usepackage[nameinlink]{cleveref} 
\usepackage[font=small,labelfont=bf,
   justification=justified,
   format=plain]{caption}
\usepackage{tikz}
\usepackage{wrapfig,booktabs}
\usepackage{xspace}
\usepackage{cancel}
\usepackage{sidecap}

\graphicspath{ {../figures/} }

\DeclarePairedDelimiterX{\infdivx}[2]{(}{)}{%
  #1\;\delimsize\|\;#2%
}

\newcommand{\trace}{\text{tr}}

\usetikzlibrary{shapes.geometric, arrows, bayesnet, calc, positioning}

\newcommand{\calA}{\mathcal{A}}
\newcommand{\calB}{\mathcal{B}}

\newcommand{\calN}{\mathcal{N}}
\newcommand{\calO}{\mathcal{O}}

\newcommand{\calX}{\mathcal{X}}
\newcommand{\calY}{\mathcal{Y}}
\newcommand{\calD}{\mathcal{D}}

\newcommand{\R}{\mathbb{R}}

\newcommand{\Pb}{\mathbb{P}}

\newcommand{\closer}[3]{{\kern-#1ex{#2}\kern-#3ex}}

\mathchardef\mhyphen="2D

\DeclareMathOperator{\E}{\mathbb{E}}
\DeclareMathOperator{\Var}{\mathbb{V}}

\newcommand\reallywidehat[1]{\arraycolsep=0pt\relax%
\begin{array}{c}
\stretchto{
  \scaleto{
    \scalerel*[\widthof{\ensuremath{#1}}]{\kern-.5pt\bigwedge\kern-.5pt}
    {\rule[-\textheight/2]{1ex}{\textheight}} 
  }{\textheight} %
}{0.5ex}\\           
#1\\                 
\rule{-1ex}{0ex}
\end{array}
}

%% file: include/custom_paper.tex
\usepackage{algorithm}

\definecolor{azure}{rgb}{0.0, 0.5, 1.0}
\definecolor{airforceblue}{rgb}{0.36, 0.54, 0.66}
\definecolor{darkgreen}{rgb}{0.0, 0.2, 0.13}

\usepackage{standalone}
\usepackage{tikz}
\usepackage{pgfplots}
\usepackage{pgf}
\usetikzlibrary{calc}
\usetikzlibrary{positioning}
\usetikzlibrary{angles,quotes}
\usetikzlibrary{backgrounds}
\usetikzlibrary{fit}
\usetikzlibrary{arrows}
\usetikzlibrary{arrows.meta}
\usetikzlibrary{shapes.symbols}
\usetikzlibrary{shadings}
\usetikzlibrary{shapes}
\usetikzlibrary{fadings}
\usetikzlibrary{bayesnet}
\usetikzlibrary{matrix}
\usetikzlibrary{plotmarks}
\usetikzlibrary{intersections}
\usetikzlibrary{pgfplots.fillbetween}
\pgfplotsset{compat=1.14}
\usepackage{siunitx}

\usepackage{colortbl}
\definecolor{mediumgray}{gray}{0.7}
\definecolor{lightgray}{gray}{0.85}
\definecolor{lightlightgray}{gray}{0.9}
\definecolor{C1}{HTML}{1F77B4}
\definecolor{C2}{HTML}{FF7F0E}
\definecolor{C3}{HTML}{2CA02C}
\definecolor{C4}{HTML}{D62728}
\definecolor{C5}{HTML}{9467BD}
\colorlet{C1light}{C1!70!white}
\colorlet{C2light}{C2!70!white}
\colorlet{C3light}{C3!70!white}
\colorlet{C4light}{C4!70!white}
\colorlet{C5light}{C5!70!white}
\colorlet{C1lighter}{C1!50!white}
\colorlet{C2lighter}{C2!50!white}
\colorlet{C3lighter}{C3!50!white}
\colorlet{C4lighter}{C4!50!white}
\colorlet{C5lighter}{C5!50!white}
\colorlet{C1vlight}{C1!20!white}
\colorlet{C2vlight}{C2!20!white}
\colorlet{C3vlight}{C3!20!white}
\colorlet{C4vlight}{C4!20!white}
\colorlet{C5vlight}{C5!20!white}
\colorlet{linkcolor}{violet}

\newtheorem{thm}{Theorem}

\newtheorem{lem}{Lemma}

\newtheorem{prop}{Proposition}

\crefname{enumi}{}{}
\crefname{enumii}{}{}

\usepackage{xr}
\makeatletter

\newcommand*{\addFileDependency}[1]{
\typeout{(#1)}
\@addtofilelist{#1}
\IfFileExists{#1}{}{\typeout{No file #1.}}
}\makeatother


%% file: icml_2024/0_abstract.tex
\begin{abstract}
Personalized decision making requires the knowledge of potential outcomes under different treatments, and confidence intervals about the potential outcomes further enrich this decision-making process and improve its reliability in high-stakes scenarios.
Predicting potential outcomes along with its uncertainty in a counterfactual world poses the foundamental challenge in causal inference. 
Existing methods that construct confidence intervals for counterfactuals either rely on the assumption of strong ignorability that completely ignores hidden confounders, or need access to un-identifiable lower and upper bounds that characterize the difference between observational and interventional distributions.
In this paper, to overcome these limitations, we first propose a novel approach wTCP-DR based on transductive weighted conformal prediction, which provides confidence intervals for counterfactual outcomes with marginal converage guarantees, even under hidden confounding.
With less restrictive assumptions, our approach requires access to a fraction of interventional data (from randomized controlled trials) to account for the covariate shift from observational distributoin to interventional distribution.
Theoretical results explicitly demonstrate the conditions under which our algorithm is strictly advantageous to the naive method that only uses interventional data.
Since transductive conformal prediction is notoriously costly, we propose wSCP-DR, a two-stage variant of wTCP-DR, based on split conformal prediction with same marginal coverage guarantees but at a significantly lower computational cost.
After ensuring valid intervals on counterfactuals, it is straightforward to construct intervals for individual treatment effects (ITEs).
We demonstrate our method across synthetic and real-world data, including recommendation systems, to verify the superiority of our methods compared against state-of-the-art baselines in terms of both coverage and efficiency.

%
%
%
\end{abstract}

%% file: icml_2024/1_introduction.tex
\section{Introduction}


\begin{figure}[t]
\begin{minipage}{0.45\textwidth}
\centering
\includegraphics[width=\linewidth]{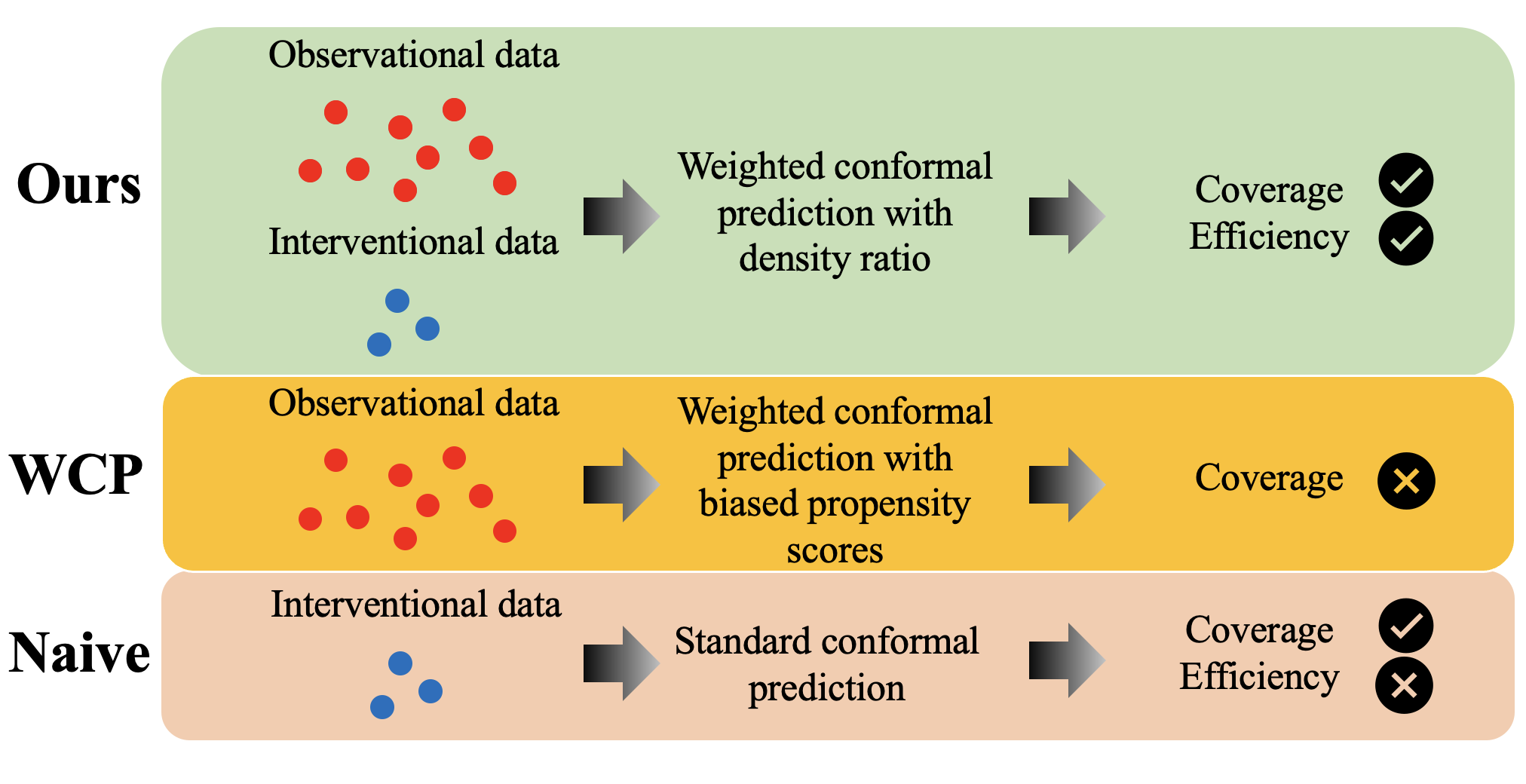}
\end{minipage}
\caption{Under hidden confounding, our proposed methods wTCP-DR and wSCP-DR incorporate a small set of interventional data for density ratio based weighted conformal prediction, which provides marginal coverage guarantee along with high efficiency (small confidence interval). In contrast, WCP~\cite{lei2021conformal} cannot guarantee coverage as hidden confounding leads to biased estimate of propensity scores. The Naive method suffers from low efficiency as it only uses the small set of interventional data.}
\label{fig:intro}
\end{figure}

Estimating the heterogeneous causal effects of an intervention (e.g., a medicine) on an important outcome (e.g., health status) of different individuals is a fundamental problem in a variety of influential research areas, including economics, healthcare and education~\cite{zhou2022attendance,wendling2018comparing, breen2015heterogeneous}.
In the growing area of machine learning for causal inference, this problem has been casted as estimating individual treatment effect (ITE) and most existing work focuses on developing machine learning models to improve the point estimate of ITE~\cite{chipman2010bart,hill2011bayesian,wager2018estimation,johansson2016learning,shalit2017estimating,louizos2017causal,kunzel2019metalearners,shi2019adapting,yao2018representation,curth2021nonparametric}.
However, point estimates is not enough to ensure safe and reliable decision-making in high-stake applications where failures are costly or may endanger human lives, and hence uncertainty quantification and confidence intervals allow machine learning models to express confidence in the correctness of their predictions.


Pioneering work~\cite{hill2011bayesian,alaa2017bayesian} provides confidence intervals for ITEs through Bayesian machine learning models such as Bayesian Additive Regression Trees~\cite{chipman2010bart} and Gaussian Process~\cite{williams2006gaussian}.
However, these approaches cannot be easily generalized to popular machine learning models for causal inference on various input data types, including but not limited to text~\cite{veitch2020adapting,feder2022causal} and graphs~\cite{guo2020learning,ma2022learning}. 
Recently, built upon conformal prediction~\cite{vovk2005algorithmic,vovk2009line}, Lei and Candes~\cite{lei2021conformal} propose the first conformal prediction method for counterfactual outcomes and ITEs, which can provide confidence intervals with guaranteed marginal coverage in a model-agnostic fashion.
This means that, given any machine learning model that estimates the potential outcomes under treatment, conformal prediction acts as a post-hoc wrapper that provides confidence intervals guaranteed to contain the ground truth of potential outcomes and ITEs above a specified probability under marginal distribution.
Unfortunately however, Lei and Candes~\cite{lei2021conformal} require the assumption of strong ignorability that excludes the possibility of hidden confounders, which cannot be verified given data~\cite{tchetgen2020introduction,pearl2018book} and can be violated in many real-world applications.
For example, the socio-economic status of a patient, which is likely to be unavailable due to privacy concerns, is a common unobserved confounding factor that affects both patient's access to treatment and one's health condition.
Similarly, under the strong ignorability assumption, \cite{alaa2023conformal} propose to use meta-learners~\cite{kunzel2019metalearners,horvitz1952generalization,kennedy2020towards} in conformal prediction of ITEs.
Recently, Jin et al.~\cite{jin2023sensitivity} take hidden confounding into consideration for conformal prediction of ITEs from a sensitivity analysis aspect.
However, their method needs access to the upper and lower bounds of the density ratio between the observational distribution and the interventional distribution to characterize the covariate shift from observational to interventional distribution.

To address these limitations and provide confidence intervals that have finite-sample guarantees even without the strong ignorability assumption, we propose \emph{\underline{w}eighted \underline{T}ransductive \underline{C}onformal \underline{P}rediction with \underline{D}ensity \underline{R}atio estimation} (wTCP-DR) that is based on weighted transductive conformal prediction.
With less restrictive assumptions, wTCP-DR needs access to both observational and a fraction of interventional data (e.g., data collected from randomized control trials)~\cite{li2021unifying,chen2021autodebias}. 
In contrast to the weighted conformal prediction method proposed by~\cite{lei2021conformal} which uses propensity score as the reweighting function, our algorithm computes the reweighting function by learning the density ratio of the interventional and observational distribution using the data provided.
The benefits of our proposed method are as follows:
(i) wTCP-DR does not require strong ignorability assumption and provides a confidence interval with coverage guarantee even under the presence of confounding. 
(ii) wTCP-DR works well under an imbalanced number of interventional and observational data, i.e., when interventional data is of smaller size than observational data due to the higher cost of collecting interventional data.
%
Although wTCP-DR is computationally expensive due to the nature of transductive conformal prediction, we also propose a variant of wTCP-DR, called \emph{\underline{w}eighted \underline{S}plit \underline{C}onformal \underline{P}rediction with \underline{D}ensity \underline{R}atio estimation} (wSCP-DR)
which preserves all the advantages of wTCP-DR but at a lower computational cost.
We briefly describe how our methods are different from the method proposed by~\cite{lei2021conformal} and the Naive method in Fig.~\ref{fig:intro}.
%



The paper is organized as follows. 
Section 2 gives a description of the problem setting and provides necessary background on conformal prediction.
Section 3 describes our novel algorithm wTCP-DR which provides a confidence interval on counterfactual outcomes at an individual level with marginal coverage guarantee.
Section 4 proposes wSCP-DR which is a more implementable variant of wTCP-DR.
Section 5 applies wTCP-DR and wTCP-DR to provide confidence intervals for estimating individual treatment effects.
Section 6 demonstrates our method across synthetic and real-world data, including recommendation systems, to verify our methods in terms of both coverage and efficiency.
Section 7 discusses related work in the literature.
Section 8 concludes the paper.

%% file: icml_2024/2_background.tex
\section{Preliminaries}
\subsection{Problem setting}

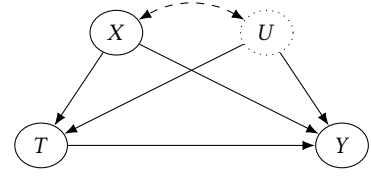
\begin{figure}[t]
\vspace{-5pt}
\begin{minipage}{0.49\linewidth}
\centering
\begin{tikzpicture}
\centering
    \node[state] (t) at (-2,0) {${T}$};
    \node[state] (y) at (2,0) {${Y}$};

    \node[state,dotted] (z) at (1,1.5) {${U}$};

    \node[state] (x) at (-1,1.5) {$X$}; 

    \path (x) edge (t);
    \path (x) edge (y);
    \path (t) edge (y);
    
    \path (z) edge (y);

    \path (z) edge (t);
    
    \path[bidirected] (z) edge[bend right=30] (x);
\end{tikzpicture}
\end{minipage}
\caption{Example causal graph with hidden confounding. $X$: Observed covariates, $U$: Hidden confounders, $T$: Treatment, $Y$: Outcome. Direct edges denote causal relations and the bidirectional edge signifies possible correlation.}
\label{fig:fig2}
\end{figure}

We consider the standard potential outcome (PO) framework~\cite{neyman1923application,rubin2005causal} with a binary treatment. Let $T \in\{0,1\}$ be the treatment indicator, $x \in \mathcal{X} \subset \R^d$ be the observed covariates, and $y \in \calY \subset \mathbb{R}$ be the outcome of interest. 
We use $X, Y$ to denote random variables in $\calX, \calY$.
For each subject $i$, let $\left(Y_i(0), Y_i(1)\right)$ be the pair of potential outcomes under control $T=0$ and treatment $T=1$, respectively.
We assume that the data generating process satisfies the following widely used assumptions: 1) Consistency: $Y_i=Y_i(T_i)$, which means the observed outcome $Y_i$ is the same as the potential outcome $Y_i(T_i)$ with the observed treatment $T_i$. (2) Positivity: $0< \Pb (T=1 \mid X=x)<1, \forall x \in \mathcal{X}$, which means that any subject has a positive chance to get treated and controlled.
We would like to emphasize that we are \emph{not} assuming strong ignorability, i.e., there might exist potential hidden confounding $U$ that affects treatment $T$ and outcome $Y$ at the same time.
See Fig.~\ref{fig:fig2} for an example causal graph.

Under this framework, the joint distribution under intervention $do(T=t)$ is $P_{X,Y(t)}=P_{Y(t)\mid X} \times P_{X}$ and that for observational data is $P_{X, Y\mid T=t} = P_{Y\mid X,T=t} \times P_{X\mid T=t}$.
Note that the difference between conditional distribution $P_{Y(t) \mid X} $ and $ P_{Y \mid X,T=t}$ is due to potential hidden confounding, and the difference between $P_X$ and $P_{X \mid T=t}$ is due to intervention.
Throughout this work, we stick to the notation of probability density (mass) functions instead of probability measures. We use superscript $I$ for interventional distribution and $O$ for observational distribution.
For a given treatment $t \in \{0, 1\}$, we assume there are $n$ observational and $m$ interventional samples: 
\begin{align}\label{eq:samples}
\begin{aligned}
    (x^{O, T=t}_i, y^{O, T=t}_i)_{i=1}^n &\sim p^O_t(x,y)= p^O(y\mid x,t )p(x\mid t) \\
    (x^{I, T=t}_i, y^{I, T=t}_i)_{i=n+1}^{n+m} &\sim p^I_t(x,y)=p^I(y\mid x,t) p(x)
\end{aligned}
\end{align}
Given a predetermined target coverage rate of $1-\alpha$, our goal is to construct confidence interval $C$ for potential outcome under treatment $t$ at a new test sample $x_{n+m+1} \sim p(x)$, such that $C(x_{n+m+1})$ ensures marginal coverage:
    $\Pb \left(y_{n+m+1} \in C(x_{n+m+1}) \right) \geq 1 - \alpha$, 
where the probability is over $(x_{n+m+1}, y_{n+m+1}) \sim p_t^I(x,y)$.

\subsection{Background: Conformal Prediction}
Conformal prediction (CP) is a distribution-free framework that provides finite-sample marginal coverage guarantees. 
Transductive and split CP are two approaches to conformal prediction and we briefly introduce both since we will be using them in 
Section 3. 

\noindent\textbf{Split Conformal Prediction (SCP).}
Given a dataset $\calD=(x_i, y_i)_{i=1}^n \sim P_{X,Y}$, SCP starts by splitting $\calD$ into two disjoint subsets: a training set $\calD_t$, and a calibration set $\calD_c$. Then, a regression estimator $\widehat{\mu}$ is trained on $\calD_t$ and conformity scores $s(x,y)$  are computed for $(x,y) \in \calD_c$ where typically $s(x,y) = |y - \widehat{\mu}(x)|$.
The empirical distribution of the conformity scores are defined as $\widehat{F} = \frac{1}{|\calD_c|} \sum_{i=1}^{|\calD_c|} \delta_{s(x_i,y_i)}$ and the confidence interval for the target sample $x_{n+1}$ is 
\begin{align}\label{eq:scp}
    C_{\text{SCP}}(x_{n+1}) = [ \widehat{\mu}(x_{n+1}) - q_{\widehat{F}}, \widehat{\mu}(x_{n+1}) + q_{\widehat{F}} ]
\end{align}
where $q_{\widehat{F}} = \operatorname{Quantile} ( (1-\alpha)(1+\frac{1}{|\calD_c|}); \widehat{F} )$. 
\cite{lei2018distribution} has proved that under exchangeability of $\calD$, $C_{\text{SCP}}(x_{n+1})$ is guaranteed to satisfy marginal coverage. 
Futhermore, if ties between conformity scores occur with probability zero, then
\begin{align}\label{eq:scp_bound}
    1-\alpha \leq \Pb \left(y_{n+1} \in C_{\text{SCP}} \left(x_{n+1} \right)\right) \leq 1 - \alpha + \frac{1}{|\calD_c|}
\end{align}
Note that the upper bound ensures that the confidence interval is nonvacuuous, i.e., the interval width does not go to infinity.

\noindent\textbf{Transductive Conformal Prediction (TCP).}
Given a same dataset $\calD$ as above, TCP takes a different approach by looping over all possible values $\overline{y}$ in the domain $\calY$.
For $\overline{y} \in \calY$, TCP first constructs an augmented dataset $\calD_{(x_{n+1}, \overline{y})} = \calD \cup \{x_{n+1}, \overline{y}\}$. Then, a regression estimator $\widehat{\mu}_{\overline{y}}$ is trained on $\calD_{(x_{n+1}, \overline{y})}$ and the conformity scores read $s^{\overline{y}}_i = |y_i - \widehat{\mu}_{\overline{y}}(x_i)|$ for $i=1, \cdots, n$ and $s^{\overline{y}}_{n+1} = |\overline{y} - \widehat{\mu}_{\overline{y}}(x_{n+1})|$. 
With empirical distribution defined as $\widehat{F} = \frac{1}{n+1} \sum_{i=1}^{n} \delta_{s^{\overline{y}}_i } + \frac{1}{n+1} \delta_{\infty}$, the interval for the target sample $x_{n+1}$ is 
\vspace{-5pt}
\begin{align}\label{eq:tcp}
    C_{\text{TCP}}(x_{n+1}) = \{ \overline{y} \in \calY: s^{\overline{y}}_{n+1} \leq q_{\widehat{F}} \}
\end{align}
where $q_{\widehat{F}} = \operatorname{Quantile}( (1-\alpha); \widehat{F})$. The same lower and upper bound guarantee as $\eqref{eq:scp_bound}$ has been proved in \cite{lei2018distribution}.

TCP is computationally more expensive as it requires fitting $\widehat{\mu}$ for every fixed $\overline{y} \in \calY$. 
The discretization of $\calY$ comes as a tradeoff between computational costs and accuracy of the conformal interval.
For these reasons, SCP is more widely used due to its simplicity, however, SCP is less sample efficient by splitting the dataset into a training set and a calibration set.
Cross-conformal prediction can be used to improve efficiency for SCP \citep{vovk2015cross}. 

\subsection{Weighted Conformal Prediction}
When calibration and test data are independent yet not drawn from the same distribution, \cite{tibshirani2019conformal} propose a weighted version of conformal prediction.
In this section, we discuss a more specific setting of \cite{tibshirani2019conformal} where the dataset are merged from two different distributions, $\calD = \{(x_i, y_i)_{i=1}^n \sim P_{X, Y} \} \cup \{ (x_i, y_i)_{i=n+1}^{n+m} \sim P^\prime_{X, Y} \}$ and the test sample $x_{n+m+1}$ is sampled from $ P^\prime_{X}$.
Define the density ratio as $r(x,y) = \frac{dP^\prime_{X,Y}}{dP_{X,Y}}(x, y)$, then $(x_i, y_i)_{i=1}^{n+m+1}$ are weighted exchangeable with weight functions $w(x,y) = 1$ if $(x,y) \sim P_{X,Y}$ and $w(x, y) = r(x,y)$ if $(x,y) \sim P_{X,Y}^\prime$.
For $\overline{y} \in \calY$, define the normalized weights $p_i$ as:
\vspace{-5pt}
\begin{align}\label{eq:normalized_weights}
    p_i = \frac{\sum\limits_{\sigma: \sigma(n+m+1)=i} \prod\limits_{j=n+1}^{n+m+1} r (x_{\sigma(j)}, y_{\sigma(j)} )}{\sum\limits_\sigma \prod\limits_{j=n+1}^{n+m+1} r (x_{\sigma(j)}, y_{\sigma(j)})}
\end{align}
where the summations are taken over permutations $\sigma$ of $1, \cdots, n+m+1$ (see \citep[Lemma 3]{tibshirani2019conformal}). 
Here in Eq.~\eqref{eq:normalized_weights}, we use an abuse of notation that $y_{n+m+1} = \overline{y}$ for symmetry reason.
With the conformity scores $s^{\overline{y}}_i$ computed in the same way as TCP and the weighted empirical distribution of the conformity scores defined as $\widehat{F} = \sum_{i=1}^{n+m} p_i \delta_{s^{\overline{y}}_i} + p_{n+m+1} \delta_{\infty}$, the conformal interval for the target sample is:
\begin{align}\label{eq:wtcp}
    C_{\text{w-TCP}}(x_{n+m+1}) = \{ \overline{y} \in \calY: s^{\overline{y}}_{n+m+1} \leq q_{\widehat{F}} \}
\end{align}
where $q_{\widehat{F}} = \operatorname{Quantile} (1-\alpha ; \widehat{F} )$. The lower bound guarantee is proven in \cite{tibshirani2019conformal} and the upper bound is proven in \cite{lei2021conformal} under extra assumptions. 
When $m=0$, $p_i$ becomes $r (x_i, y_i ) / \sum_{j=1}^{n+1} r (x_j, y_j )$, which is more commonly used in the literature~\cite{taufiq2022conformal, lei2021conformal, kunzel2019metalearners}. 
When $m > 1$, the computational cost of $p_i$ is $m C_{n+m+1}^m = \calO(mn^m)$. 

%% file: icml_2024/3_methodology.tex
\section{Conformal Prediction of Counterfactuals: WTCP-DR}
In this section, we formally introduce our proposed method \emph{\underline{w}eighted \underline{T}ransductive \underline{C}onformal \underline{P}rediction with \underline{D}ensity \underline{R}atio estimation} (wTCP-DR).
Since our method considers $T=0$ and $T=1$ separately, we fix $T=t$ in this section and drop the dependence on $T$ in Eq.~\eqref{eq:samples} for simplicity of notations.
Recall there are $n$ observational and $m$ interventional samples and the test sample is $x_{n+m+1}$. 
\begin{align}
\begin{aligned}
    (x^{O}_i, y^{O}_i)_{i=1}^n &\sim p^O(x,y)= p^O(y\mid x,t )p(x\mid t) \\
    (x^{I}_i, y^{I}_i)_{i=n+1}^{n+m} &\sim p^I(x,y) = p^I(y\mid x,t) p(x)
\end{aligned}
\end{align}

\noindent\textbf{The Naive Method.}\customlabel{naive_approach}{naive approach}
We first introduce a straightforward method: constructing confidence interval for the potential outcome \emph{only} from interventional data $(x^{I}_i, y^{I}_i)_{i=n+1}^{n+m}$ using standard split conformal prediction of Eq.~\eqref{eq:scp} as $(x^{I}_i)_{i=n+1}^{n+m}$ come from the same distribution as the test sample $x^{I}_{n+m+1}$. 
The algorithm is detailed in \Cref{alg:naive}.
From Eq.~\eqref{eq:scp_bound} we know that
\begin{align}\label{eq:naive_coverage}
    1 - \alpha + \frac{1}{m + 1} \geq \Pb(y \in C_{\text{naive}}(x)) \geq 1 - \alpha
\end{align}
This approach can be inefficient because it completely ignores $n$ observational data and typically $n $ is larger than $ m$.

\begin{algorithm}
\caption{Naive algorithm}
\label{alg:naive}
\begin{algorithmic}[1] 
\Require level $\alpha$, interventional data $\mathcal{D}^I=(x_i^I, y_i^I )_{i=n+1}^{n+m}$ split into a training fold $\mathcal{D}^I_t$ and a calibration fold $\mathcal{D}^I_c$, target sample $x_{n+m+1}^I$.
\State Fit regression model $\hat{\mu}$ on $\mathcal{D}^I_t$.
\For{each sample $(x_i, y_i) \in \mathcal{D}^I_c$}
    \State Compute the conformity score $s_i = |\hat{\mu} (x_i) - y_i|$.
\EndFor
\State Construct empirical distribution of conformity scores $\widehat{F} = \frac{1}{|\mathcal{D}^I_c|} \sum_{i=1}^{\left|\mathcal{D}^I_c\right|} \delta_{s_i}$.
\State Compute $q_{\widehat{F}} = \operatorname{Quantile} ( (1-\alpha)(1 + \frac{1}{|\mathcal{D}_c|}); \widehat{F})$.
\Ensure $C_{naive}(x_{n+m+1}^I) = [\hat{\mu}(x_{n+m+1}^I)-q_{\widehat{F}}, \hat{\mu}(x_{n+m+1}^I) + q_{\widehat{F}}]$
\end{algorithmic}
\end{algorithm}

To combine both $m$ interventional data and $n$ observational data, it is necessary to take distribution shift into consideration. 
Therefore, weighted conformal prediction of Eq.~\eqref{eq:wtcp} is naturally suitable for such tasks, and the key challenge is to identify the normalized weights in Eq.~\eqref{eq:normalized_weights}, i.e., to identify the density ratio
\begin{align}\label{eq:density_ratio}
    r(x,y) := \frac{p^I(x,y)}{p^O(x,y)} = \frac{p^I(y\mid x,t) p(x) }{ p^O(y\mid x,t ) p(x\mid t)}
\end{align}

Under the unconfoundedness assumption of \cite{lei2021conformal}, $p^I(y\mid x,t)$ equals $p^O(y\mid x,t )$ so $r(x,y)$ is as simple as estimating the propensity score $p(x) / p(x \mid t)$.
When hidden confouding exists, propensity score is not enough to account for the distribution shift. Our method proposes to learn $r(x,y)$ from data, as detailed next.

\noindent\textbf{Weighted Transductive Conformal Prediction with Density Ratio estimation (wTCP-DR).} \customlabel{our_approach}{our approach}
The key of weighted conformal prediction is the density ratio $r(x,y)$, and fortunately there exists a rich literature of density ratio estimation~\cite{sugiyama2012density}, including moment matching~\cite{gretton2009covariate}, probabilistic classification and ratio matching. 
Since probabilistic classification using neural networks is more flexible and better exploits nonlinear relations in the data~\cite{yamada2013relative}, so we only introduce probabilistic classification here and refer the readers to \cite{sugiyama2012density} for a comprehensive review.

By assigning labels $z=1$ to observational data $(x^O_i, y^O_i)$ and assigning labels $z=0$ to interventional data $(x^I_i, y^I_i)$, we construct a new dataset for learning the density ratio.
\begin{align*}
    \calD_{\text{DR}} = \{(x^{O}_i, y^{O}_i, z_i)_{i=1}^n,  (x^{I}_i, y^{I}_i, z_i)_{i=n+1}^{n+m}\}
\end{align*}
For any nonlinear binary classification algorithm like logistic regression with nonlinear features, random forests or neural networks that output estimated probabilities of class membership $\hat{p}(z=1\mid x, y)$ and $\hat{p}(z=0\mid x, y)$, the density ratio can be approximated by:
\begin{align}
\begin{aligned}
    \frac{p^I(x,y)}{p^O(x,y)} &= \frac{p(x,y\mid z = 0)}{p(x,y \mid z=1)} = \frac{p(z = 0 \mid x, y) / p(z=0)}{p( z = 1\mid x, y) / p(z=1)} \\
    &\approx \frac{p(z=1)}{p(z=0)} \frac{\hat{p}(z=0\mid x, y)}{\hat{p}(z=1\mid x, y)}
\end{aligned}
\end{align}
Since $\frac{p(z=1)}{p(z=0)} $ is a constant and will cancel out when computing the normalized weights in Eq.~\eqref{eq:normalized_weights}, we 
denote $\hat{r}(x,y)= \frac{\hat{p}(z=0\mid x, y)}{\hat{p}(z=1\mid x, y)}$ as the estimated density ratio, so the corresponding estimated normalized weights of Eq.~\eqref{eq:normalized_weights} are:
\begin{align}\label{eq:estimated_normalized_weights}
\hat{p_i} = \frac{\sum\limits_{\sigma: \sigma(n+m+1)=i} \prod\limits_{j=n+1}^{n+m+1} \hat{r} (x_{\sigma(j)}, y_{\sigma(j)} )}{\sum\limits_\sigma \prod\limits_{j=n+1}^{n+m+1} \hat{r} (x_{\sigma(j)}, y_{\sigma(j)})}
\end{align}
Unfortunately, Eq.~\eqref{eq:estimated_normalized_weights} requires $m C_{n+m+1}^m = \calO(mn^m)$ times of evaluating $\hat{r}$ which is computationally impractical for $m > 1$. 
As a result, we only use observational data when computing the normalized weights (i.e. $m=1$) and use interventional data for computing the density ratio $\hat{r}$, so the estimated normalized weights become
\begin{align}\label{eq:easy_estimated_normalized_weights}
    \hat{p_i} = \frac{ \hat{r}(x_i, y_i) }{\sum_{j=1}^n \hat{r}(x_j, y_j) + \hat{r}(x_{n+m+1}, y_{n+m+1})}
\end{align}
for $i=\{1 , \cdots, n\} \cup \{ n+m+1 \}$.
See \Cref{alg:wtcp_dr} for a complete description of our method.

By using estimated normalized weights $\hat{p_i}$ rather than the oracle normalized weights $p_i$ to reweight the empirical distribution of conformity scores $\widehat{F}$, our approach introduces an extra source of error, as quantified below.
\begin{prop}[Prosample 4.2 from \cite{taufiq2022conformal}] \label{prop:ours_coverage}
Under the assumptions that $p^O(x,y)$ and $ p^I(x,y)$ are absolutely continuous with each other and that ${ [\E_{p^O(x,y)} \hat{r}(x,y)^2]}^{1/2} < M$ then the confidence interval $C_{\text{wTCP-DR}}$ constructed from \Cref{alg:wtcp_dr} satisfies
\begin{align}\label{eq:ours_coverage}
\begin{aligned}
    &1 - \alpha + c n^{-1/2} + \Delta_r \geq \Pb \left( y \in C_{\text{wTCP-DR}}(x) \right) \geq 1 - \alpha - \Delta_r
\end{aligned}
\end{align}
where $c$ is a constant and $\Delta_r = \E_{p^O(x,y)} |r(x,y) - \hat{r}(x, y)| $ is the approximation error of the density ratio.
\end{prop}

\setlength{\textfloatsep}{0pt}
\begin{algorithm}[h]
\caption{Weighted Transductive Conformal Prediction with Density Ratio Estimation (wTCP-DR)}
\label{alg:wtcp_dr}
\begin{algorithmic}[1] 
\Require level $\alpha$, observational data $\mathcal{D}^O=(x_i^O, y_i^O )_{i=1}^n$ and interventional data $\mathcal{D}^I = (x_i^I, y_i^I )_{i=n+1}^{n+m}$, test sample $x_{n+m+1}^I$.
\State Initialize $C_{\text{wTCP-DR}}(x_{n+m+1}^I) = \varnothing$.
\State Estimate the density ratio $\hat{r}$ using $\mathcal{D}^O$ and $\mathcal{D}^I$.
\For{$\overline{y} \in \mathcal{Y}$}
    \State Construct augmented dataset $\mathcal{D}_{\overline{y}} = \mathcal{D}^O \cup \{ x_{n+m+1}^I, \overline{y}\}$.
    \State Fit a regression model $\hat{\mu}$ on $\mathcal{D}_{\overline{y}}$.
    \State Compute conformity scores $s^{\overline{y}}_i = |\hat{\mu}(x_i^O) - y_i^O|$ for $i=1, \cdots, n$ and $s^{\overline{y}}_{n+m+1} = |\hat{\mu}(x^I_{n+m+1}) - \overline{y}|$.
    \State Compute the normalized weights $\hat{p}_i$ as in Eq.~\eqref{eq:easy_estimated_normalized_weights} ($y_{n+m+1}$ is replace with $\overline{y}$).
    \State Construct weighted empirical distribution of conformity scores $\widehat{F} = \sum_{i=1}^{n} \hat{p}_i \delta_{s^{\overline{y}}_i} + \hat{p}_{n+m+1} \delta_{\infty}$.
    \State Compute quantile $q_{\widehat{F}} = \operatorname{Quantile} (1-\alpha; \widehat{F})$.
    \If{$s_{n+m+1} \leq q_{\widehat{F}}$}
        \State $C_{\text{wTCP-DR}}(x_{n+m+1}^I) = C_{\text{wTCP-DR}}(x_{n+m+1}^I) \cup \{ \overline{y} \}$.
    \EndIf
\EndFor
\Ensure $C_{\text{wTCP-DR}}(x_{n+m+1}^I)$.
\end{algorithmic}
\end{algorithm}

By comparing Eq.~\eqref{eq:ours_coverage} and Eq.~\eqref{eq:naive_coverage}, we can see that when we have access to the oracle density ratio $r(x,y)$, i.e $\Delta_r = 0$, then wTCP-DR obtains a tighter upper bound than the naive method, as typically the number of observational data $n$ is much larger than the number of interventional data $m$ in causal inference, due to the higher cost of randomized controlled trails.
Unfortunately, oracle density ratio $r(x,y)$ is usually unavailable, and the estimation error of density ratio is of order $\min (n, m)^{-1/2} = m^{-1/2}$ for moment matching or ratio matching~\cite{sugiyama2012density, yamada2013relative} and of order $m^{-1/2}$ for probabilistic classification~\cite{bartlett2006convexity}. 
It seems that wTCP-DR has spent a huge amount of effort while achieving a worse result in the end.

However, we would like to emphasize that the efficiency of conformal prediction methods is quantified by the \textbf{width of the confidence interval}, not by the difference between the probability upper and lower bound.
An upper bound strictly lower than $1$ guarantees that the confidence interval is not arbitrarily large, however there is no guarantee that a smaller upper bound results in a smaller confidence interval.
Intuitively, our method has a smaller interval compared to the naive method, because the regression model $\hat{\mu}$ of wTCP-DR is trained on $n$ observational data while the regression model $\hat{\mu}$ of naive method is trained on $m$ interventional data. Intuitively, there is a higher chance that the conformity scores of wTCP-DR are smaller than the conformity scores of the naive method, which means that $C_{\text{wTCP-DR}}$ is a smaller interval than $C_{\text{naive}}$. 
We formalize the above intuition in the following section for additive Gaussian noise model.

\subsection{Case Study: Additive Gaussian Noise Model}
\label{subsec:case_study}
In this section, we consider an additive Gaussian noise model, which is a simple yet popular setting in causal inference~\cite{hoyer2008nonlinear}.
Recall that we fix $T=t$ and
drop the dependence on $T$ for simplicity of notations. 
Specifically, we make the following assumptions:
\begin{enumerate}[itemsep=0.1pt,topsep=0pt,leftmargin=*]
\item [A1] Additive Gaussian noise. 
$y^O \sim \calN({\theta^O}^\top \varphi(x^O), \sigma^2) $ and $y^I \sim \calN({\theta^I}^\top \varphi(x^I), \sigma^2)$, where $\varphi$ represents the (learned) features of interventional and observational data.
%

\item [A2] Gaussian features. $\varphi(x^O) \sim \calN(0, \Sigma^O)$ and $\varphi(x^I) \sim \calN(0, \Sigma^I)$.
\item [A3] Upper bounds on the difference between oracle density ratio $r(x,y)$ and estimated density ratio $\hat{r}(x,y)$.
\begin{align*}
    &\E_{p^O(x,y)} \left( r(x,y) - \hat{r}(x, y)\right)^2 < \infty \\
    \Delta_r := &\E_{p^O(x,y)} |r(x,y) - \hat{r}(x, y)| < \frac{1 - \alpha}{\alpha}
\end{align*}
\item [A4] Bounded $\chi^2$ divergence between $p^I(x,y)$ and $p^O(x,y)$.
\begin{align*}
    \chi^2(p^I\| p^O) = \int \left( \frac{p^I(x,y)}{p^O(x,y)} - 1\right)^2 p^O(x,y) dxdy < \infty
\end{align*}
\end{enumerate}
Under these assumptions, the effect of hidden confounding is reflected from the difference of $p^O(y \mid x, t)$ $p^I(y \mid x, t)$ through the difference of $\theta^O$ and $\theta^I$: $\theta^O$ is dependent of hidden confounding $u$ whereas $\theta^I$ is independent of $u$ due to intervention.
Before showing our main theoretical result, let us first discuss the implications of these assumptions.
\begin{enumerate}[itemsep=0.1pt,topsep=0pt,leftmargin=*]
\item [A1] We assume that interventional and observational data share the same feature $\varphi$, a commonly used setting in causal inference especially when $\varphi$ is learned with neural networks~\cite{shi2019adapting}.
We assume the same noise scale for observational and interventional data only for simplicity, which can be relaxed to the more general case that $y^O$ and $y^I$ have different noise scales $\sigma^O, \sigma^I$. 
\item [A2] 
This assumption is satisfied when either the features are designed to have Gaussian distribution, or the features are learned from wide enough neural networks~\cite{lee2017deep}. 
\item [A3] This assumption requires that the error of density ratio estimation is upper bounded, and given that $\alpha$ is typically $0.1$ or $0.05$, this assumption is usually satisfied in practice.
\item [A4] This assumption ensures that $p^I$ and $p^O$ share the same support over $\calX \times \calY$, and is required such that the central limit theorem can be used in the proof.
\end{enumerate}
Now we give the main theoretical result of this paper.
\begin{thm}\label{thm:main}
Assume the above assumptions hold, with probability at least $1 - \delta_1 - \delta_2 - \delta_3 - \delta_4$, the interval $C_{\text{wTCP-DR}}(x_{n+m+1}^I)$ obtained from \Cref{alg:wtcp_dr} will be smaller than the interval $C_{\text{naive}}(x_{n+m+1}^I)$ obtained from \Cref{alg:naive} up to $\calO(\sqrt{log n / n})$, with $\delta_1, \delta_2, \delta_3, \delta_4$ being the following:
\vspace{-5pt}
\begin{align*}
    &\delta_1 = \left( \frac{2}{n} \frac{1 - \alpha - \frac{\Delta_r}{\Delta_r + 1}}{\alpha + \frac{\Delta_r}{\Delta_r + 1}} \frac{p^O(x)}{p^I(x)} \right)^{ 4 \sigma^2 \sqrt{\frac{C_1}{C_2}} } , \delta_2 = \frac{2}{n}, \\
    &\delta_3 = \exp \left(-\frac{1}{2} L_{1-\alpha}^2 \left(\operatorname{erf}^{-1}(1-\alpha)\right)^2 \frac{(d-1)^2 }{m-1} \right),\\
    &\delta_4 = \exp \left(- C_\alpha^2 \frac{ n_{\text{eff}} }{(m-d)^2} \right)
\end{align*}
where $\frac{C_1}{C_2} = \frac{(\theta^I + \theta^O)^\top \Sigma^I (\theta^I + \theta^O)}{(\theta^I - \theta^O)^\top \Sigma^I (\theta^I - \theta^O)}$ represent the dissimilarity distance between $\theta^I$ and $\theta^O$; $\operatorname{erf}^{-1}$ is the inverse error function~\cite{abramowitz1948handbook}, $L_{1-\alpha}$ and $C_\alpha$ are constants that only depend on $\alpha$; and $n_{\text{eff}}$ is the effective sample size defined as below
\begin{align}
    n_{\text{eff}} = \left(\sum_{i=1}^n \hat{r} (x_i^O )\right)^2 \Big/ \sum_{i=1}^n  \hat{r} (x_i^O )^2
\end{align}
\end{thm}
The proof of \Cref{thm:main} can be found in \Cref{appsec:proof_thm1}.
The implications of \Cref{thm:main} can be summarized as below.
\begin{enumerate}[itemsep=0.1pt,topsep=0pt,leftmargin=*]
\item $\delta_1$ quantifies the number of observational data needed to contain sufficient information about the interventional distribution. If $\theta^I$ and $\theta^O$ are very close, which means that the distributions $p^I(x,y)$ and $p^O(x,y)$ are very similar, the exponenet $\frac{C_1}{C_2}$ is bigger so fewer observational data (smaller $n$)  would contain sufficient information of the interventional distributions.
\item $\delta_2$ quantifies the stability of the estimator used. Since we are using the least squared estimator which is known to be stable when $n > d$ and $m > d$, having more $n$ would entail smaller $\delta_2$.
\item $\delta_3$ and $\delta_4$ quantifies the ratio of the effective sample size $n_{\text{eff}}$ and the interventional sample size $m$.
$n_{\text{eff}}$ was first defined by \cite{gretton2009covariate} in covariate shift literature and \cite{tibshirani2019conformal} gives an intuition that the performance of weighted conformal prediction should depend on $n_{\text{eff}}$, our theorem is the first to quantitatively show that $n_{\text{eff}}$ rather than $n$ is the key to measure the performance of weighed conformal prediction when compared against standard conformal prediction.
\end{enumerate}

From \Cref{thm:main}, we can see that our method in \Cref{alg:wtcp_dr} is more efficient than the naive method in \Cref{alg:naive} in terms of width of confidence interval provided, when the interventional distribution is close to the observational distribution, when the dimension $d$ is relatively high compared to the number of interventional data $m$, and when the effective sample size $n_{\text{eff}}$ is larger than $m$.
The theoretical result is further corroborated by empirical findings in Section 6.
%
\setlength{\textfloatsep}{5pt}
\begin{algorithm}[h]
\caption{Two-stage wSCP-DR (Inexact)}
\label{alg:wscp_dr_inexact}
\begin{algorithmic}[1] 
\Require Level $\alpha$, observational data $\mathcal{D}^O=(x_i^O, y_i^O )_{i=1}^n$ and interventional data $\mathcal{D}^I = (x_i^I, y_i^I )_{i=n+1}^{n+m}$, test sample $x_{n+m+1}^I$.  
\State Use $\mathcal{D}^O$ and $\mathcal{D}^I$ to estimate the density ratio $\hat{r}$. \\
\# First stage.
\For{$x_j^I, y_j^I \in \mathcal{D}^I$}
    \State Fit a regression model $\hat{\mu}$ on $\mathcal{D}^O \cup (x_j^I, y_j^I)$.
    \State Compute conformity scores $s_i = |\hat{\mu}(x_i^O) - y_i^O|$.
    \State Compute the normalized weights $\hat{p}_i$ as in Eq.~\eqref{eq:easy_estimated_normalized_weights}.
    \State Construct weighted empirical distribution of conformity scores $\widehat{F} = \sum_{i=1}^{n} \hat{p}_i \delta_{s_i} + \hat{p}_{j} \delta_{\infty}$.
    \State Compute quantile $q_{\widehat{F}} = \operatorname{Quantile} (1-\alpha; \widehat{F})$.
    \State $C_j^L = \hat{\mu}(x_{j}^I)-q_{\widehat{F}} $ and $C_j^R = \hat{\mu}(x_{j}^I) + q_{\widehat{F}}$
\EndFor
\\ \# Second stage.
\State Fit regressor $\hat{m}^L$ on $(x_{n+1}^I, C_{n+1}^L), \cdots, (x_{n+m}^I, C_{n+m}^L)$, and fit regressor $\hat{m}^R$ on $(x_{n+1}^I, C_{n+1}^R), \cdots, (x_{n+m}^I, C_{n+m}^R)$.
\Ensure 
$C_{wSCP-DR}^{Inexact}(x_{n+m+1}^I) = [\hat{m}^L(x_{n+m+1}^I), \hat{m}^R(x_{n+m+1}^I)]$
\end{algorithmic}
\end{algorithm}

%% file: icml_2024/4_two_stage.tex
\section{Practical Algorithm: WSCP-DR}
In practice, although transductive conformal prediction in \Cref{alg:wtcp_dr} is theoretically well-grounded, it is notoriously expensive to compute, compared to split conformal prediction.
The reason that split conformal prediction cannot be used in \Cref{alg:wtcp_dr} is the density ratio $\hat{r}$ evaluated at test sample, which requires the knowledge of both test covariate $x_{n+m+1}$ and test target value $y_{n+m+1}$ but unfortunately $y_{n+m+1}$ is inaccessible to us.
In this section, we show that we can do two-stage split conformal prediction which is computationally more efficient than transductive conformal prediction \Cref{alg:wtcp_dr} and achieves the same marginal coverage guarantee.

In the first stage, recall that interventional labels $y_{n+1}^I, \cdots, y_{n+m}^I$ are accessible, so the density ratios
$\hat{r}(x_{n+1}^I, y_{n+1}^I), \cdots, \hat{r}(x_{n+m}^I, y_{n+m}^I)$ 
and the normalized conformal weights in Eq.~\eqref{eq:normalized_weights} can be computed for $n+1, \cdots, n+m$.
Therefore, split weighted conformal prediction can be used to construct intervals $(C_{n+1}^L, C_{n+1}^R), \cdots, (C_{n+m}^L, C_{n+m}^R)$ for interventional data $(x_{n+1}^I, y_{n+1}^I), \cdots, (x_{n+m}^I, y_{n+m}^I)$ with marginal coverage guarantee. 
In the second stage, by noticing that the test sample $x_{n+m+1}^I$ shares the same distribution as $x_{n+1}^I, \cdots, x_{n+m}^I$, a standard split conformal prediction can be used to construct confidence interval $[C_{x_{n+m+1}}^L, C_{n+m+1}^R]$ for the test sample $x_{n+m+1}$ with marginal coverage guarantee.
Details of this method are presented in Algorithm~\ref{alg:wscp_dr_exact}.
Additionally, we can further reduce the computational cost of Algorithm~\ref{alg:wscp_dr_exact} by directly fitting a regressor $\hat{\mu}^L$ over the interval lower bounds $(x_{n+1}^I, C_{n+1}^L), \cdots, (x_{n+1}^I, C_{n+m}^L)$ and fitting a regressor $\hat{\mu}^R$  over the interval upper bounds $(x_{n+1}^I, C_{n+1}^R), \cdots, $ $(x_{n+1}^I, C_{n+m}^R)$ in the second stage.
Therefore, we call Algorithm~\ref{alg:wscp_dr_exact} the exact two-stage method which has marginal coverage guarantee and call Algorithm~\ref{alg:wscp_dr_inexact} 
the inexact two-stage method which does not have marginal coverage guarantee but is more efficient.

\begin{algorithm}[h]
\caption{Two-stage wSCP-DR (Exact)}
\label{alg:wscp_dr_exact}
\begin{algorithmic}[1] 
\Require Level $\alpha$, observational data $\mathcal{D}^O=(x_i^O, y_i^O )_{i=1}^n$ and interventional data $\mathcal{D}^I = (x_i^I, y_i^I )_{i=n+1}^{n+m}$, test sample $x_{n+m+1}^I$.  
\State Use $\mathcal{D}^O$ and $\mathcal{D}^I$ to estimate the density ratio $\hat{r}$. \\
\# First stage.
\State Same as the first stage in \Cref{alg:wscp_dr_inexact}
\\ \# Second stage.
\State Split $\calD^I$ into a training set of size $m_1$: $\calD^I_{tr} = (x_i^I, y_i^I )_{i=n+1}^{n+m_1}$ and calibration set of size $m - m_1$: $\calD^I_{cal} = (x_i^I, y_i^I )_{i=m_1 + 1}^{m}$. 
\State Fit regressor $\hat{m}^L$ on $(x_{n+1}^I, C_{n+1}^L), \cdots, (x_{n+m_1}^I, C_{n + m_1}^L)$ and $\hat{m}^R$ on $(x_{n+1}^I, C_{n+1}^R), \cdots, (x_{n+m_1}^I, C_{n + m_1}^R)$.
\State Compute conformity scores on $\calD^I_{cal}$: $s_i = \max \{ \hat{m}^L(x_i^I) - C_{i}^L, C_{i}^R - \hat{m}^R(x_i^I)\}$ for $i = \{m_1 + 1, \cdots, m\}$.
\State Construct empirical distribution of conformity scores $\widehat{F} = \frac{1}{m-m_1} \sum_{i=m_1 + 1}^{m} \delta_{s_i}$.
\State Compute $q_{\widehat{F}} = \operatorname{Quantile} ( (1-\alpha)(1 + \frac{1}{m - m_1 }); \widehat{F} )$.
\small
\Ensure 
$C_{wSCP-DR}^{Exact}(x_{n+m+1}^I) = [\hat{m}^L(x_{n+m+1}^I) - q_{\widehat{F}}, \hat{m}^R(x_{n+m+1}^I) + q_{\widehat{F}}]$
\end{algorithmic}
\end{algorithm}
\vspace{-15pt}
\section{Conformal Inference of Individual Treatment Effect}
In Section 3 and 4, we focus on conformal inference for counterfactual outcomes $Y(1)$ and $Y(0)$. 
However, offering confidence intervals for individual treatment effects may hold greater practical significance.
Our algorithms wTCP-DR and wSCP-DR can predict confidence intervals $[C_t^L(x_{n+m+1}^I), C_t^R(x_{n+m+1}^I)], t\in\{0,1\}$ that has marginal coverage guarantee for the potential outcome $y_{n+m+1}$ under treatment $t=1$ (or under control $t=0$).
The naive way of construcing intervals for ITE is to use bonferroni correction, i.e., $C_{ITE}^L = C_1^L - C_0^R$ and $C_{ITE}^R = C_1^R - C_0^L$. We demonstrate the empirical result using the naive way in Section 6 for fair comparison among methods that infer counterfactual outcomes, and we also include the results in Appendix~\ref{subsec:app_exp_syn} where intervals for ITE are constructed using the nested methods from \citep[Section 4]{lei2021conformal}.

%% file: icml_2024/5_experiments.tex
\section{Experiments}

\subsection{Experiment on Synthetic Data}


Here, we conduct experiments for counterfactual outcome and ITE estimation on synthetic data with hidden confounding and focus on the setting where the number of observational data $n $ is larger than the number of interventional data $m$.
We aim to answer the following research questions:
\textbf{RQ1}: Can our proposed methods achieve the specified level of coverage (0.9) for potential outcomes under the setting with hidden confounding and $n $ larger than $ m$ for counterfactual outcomes and ITEs?
\textbf{RQ2}: Can our proposed methods have better efficiency (smaller confidence interval) than the Naive method which only uses interventional data?
\textbf{RQ3}: How does hidden confounding strength impact the coverage of our methods?
\textbf{RQ4}: How does the size of interventional data ($m$) impact the efficiency of our methods?

\begin{table}[htb!]
\vspace{-10pt}
\caption{Description for synthetic data, Yahoo and Coat}
\label{tab:dataset}
\vspace{-10pt}
\small
\begin{tabular}{l|c|c|c|c|c}
\hline
Dataset   & $n_{tr}$ & $n_{cal}$ & $m_{tr}$ & $m_{cal}$ & $m_{ts}$ \\ \hline
Synthetic & 5,000    & 5,000     & 125      & 125       & 200      \\ \hline
Yahoo     & 103,343  & 25,706    & 10,800   & 10,800    & 32,399   \\ \hline
Coat      & 5,568    & 1,385     & 928      & 928       & 2,784    \\ \hline
\end{tabular}
\vspace{-5pt}
\end{table}

\begin{table*}[t]
\centering
\caption{Results for counterfactual outcomes and ITEs on the synthetic data. We compare our methods wSCP-DR (Inexact), wSCP-DR (Inexact), and wTCP-DR with baselines. Results are shown for coverage and confidence interval width on the synthetic data with $n=10,000$ and $m=250$.
Boldface and underlining are used to highlight the top and second-best interval width among the methods with coverage close to 0.9.}
\label{table:main_res_cf_outcome}
\vspace{-10pt}
\small
\begin{tabular}{l|c|c|c|c|c|c}
\hline
Method & Coverage $Y(0)$ \textuparrow
 & Interval Width $Y(0)$ \textdownarrow
 & Coverage $Y(1)$ \textuparrow
 & Interval Width $Y(1)$ \textdownarrow
 & Coverage ITE \textuparrow &  Interval Width ITE \textdownarrow
 \\
\hline
wSCP-DR(Inexact) & \(0.891 \pm 0.026\) & \(\underline{0.414} \pm 0.008\) & \(0.889 \pm 0.019\) & \(\textbf{0.421} \pm 0.013\) & \(0.942 \pm 0.017\) & \(\textbf{0.835} \pm 0.016\)\\
wSCP-DR(Exact) & \(0.934 \pm 0.026\) & \(0.496 \pm 0.010\) & \(0.935 \pm 0.023\) & \(\underline{0.503} \pm 0.010\) & \(0.957 \pm 0.018\) & \(0.998 \pm 0.015\) \\
wTCP-DR & \(0.899 \pm 0.028\) & \(\textbf{0.386} \pm 0.013\) & \(0.923 \pm 0.015\) & \(0.576 \pm 0.066\) & \(0.953 \pm 0.015\) & \(\underline{0.962} \pm 0.074\)\\
WCP & \(0.572 \pm 0.039\) & \(0.222 \pm 0.007\) & \(0.608 \pm 0.042\) & \(0.227 \pm 0.009\) & \(0.710 \pm 0.027\) & \(0.449 \pm 0.012\)\\
Naive & \(0.932 \pm 0.018\) & \(0.508 \pm 0.042\) & \(0.930 \pm 0.023\) & \(0.560 \pm 0.049\) & \(0.952 \pm 0.018 \) & \(1.068 \pm 0.098\) \\
\hline
\end{tabular}
\vspace{-15pt}
\end{table*}

\noindent\textbf{Dataset.} For synthetic data, we use the following data-generating process for the observables $X, T, Y$ with hidden confounding $U$.
\begin{align}\label{eq:DGP}
\vspace{-20pt}
\begin{aligned}
    U, Z & \sim \mathcal{N}(\mathbf{0},\mathbf{I}), \epsilon_1, \epsilon_0 \sim \mathcal{N}(0,1) \\
    X & = Z \odot (a^2(1-U)+b^2 U) + U  \\
    \rho & = c \bar{U} + (1-c)(1 - \bar{U}) , \quad T \sim \text{Bern}(\rho) \\
    Y(1) & = \frac{1}{1+\exp(-3(\bar{U} + 2))} + 0.1\epsilon_1 \\ 
    Y(0) & = \frac{1}{1+\exp(-3(\bar{U} - 2))} + 0.1\epsilon_0 \\
    Y & = T Y(1) + (1-T) Y(0) 
\end{aligned}
\end{align}
$\mathbf{I}$ is $d \times d$ identity matrix, $d$ is the dimensionality of $X$, $\odot$ is the hadamard product, $\bar{U}$ is the mean of each dimension of $U$, and $a=5, b=3, c=0.9$.
When $c$ is close to $1$, $\rho$ is close to $0$ as $\bar{U}$ is close to $0$, leading to more controlled samples (less treated samples) in the observational data.

\noindent\textbf{Baselines.} \textit{Naive}: it uses interventional data for standard split conformal prediction, as detailed in Algorithm~\ref{alg:naive}.
\textit{WCP}: the algorithm proposed in~\cite{lei2021conformal} that uses propensity score as the reweighting function in WCP. 
For all the methods we use the same Gradient Boosting Tree from scikit-learn as the base model $\hat{\mu}$.

\noindent\textbf{Data Splitting Details.}
We split the observational and interventional data into training $\mathcal{D}^O_{tr},\mathcal{D}^I_{tr}$, calibration $\mathcal{D}^O_{cal},\mathcal{D}^I_{cal}$, and test $\mathcal{D}_{ts}$.
%
%
For the Naive method, we train the base model $\hat{\mu}$ on $\mathcal{D}^I_{tr}$ and compute conformity scores on $\mathcal{D}^I_{cal}$.
For WCP, we train the base model $\hat{\mu}$ on $\mathcal{D}^O_{tr}$ and compute conformity scores on $\mathcal{D}^O_{cal}$.
The propensity model is trained on $\mathcal{D}^O_{tr}$.
For our methods, we train the base model $\hat{\mu}$ on $\mathcal{D}^O_{tr}$ and compute conformity scores on $\mathcal{D}^O_{cal}$.
The density ratio estimator $\hat{r}$ is trained on $\mathcal{D}^O_{tr} \cup \mathcal{D}^I_{tr}$.
The size of each split can be found in Table~\ref{tab:dataset}.

\noindent\textbf{Evaluation Metrics.}
We use the evaluation metrics from~\cite{lei2021conformal,alaa2023conformal} for both counterfactual outcomes and ITEs. \textit{Coverage} measures the probability of the true counterfactual outcome falling in predicted confidence interval 
, where $\mathds{1}$ is the indicator function.
\textit{Interval width} is the average size of the confidence interval ${C}(x_i)$ on test samples $i \in \mathcal{D}_{ts}$, which represents the efficiency of conformal inference methods.

\noindent\textbf{Comparison Results (RQ1-2).} 
Table~\ref{table:main_res_cf_outcome} shows results under the setting of $n=10,000$ and $m=250$ under strong hidden confounding ($d=1$).
We make the following observations:
\begin{itemize}
[itemsep=0.1pt,topsep=0pt,leftmargin=*]
\item In terms of coverage, our methods wSCP-DR (Exact) and wTCP-DR achieve the specified level of coverage ($0.9$) for $Y(0)$, $Y(1)$ and ITE. wSCP-DR (Inexact) has coverage slightly lower than 0.9 for $Y(1)$ and $Y(0)$ as it trades coverage guarantee for lower computational cost. 
The coverage results verify that our proposed reweighting function based on density ratio estimation can accurately adapt the conformity scores computed on observational data to the interventional distribution even under hidden confounding.
In contrast, coverage of WCP is much lower than $0.9$, because WCP does not take hidden confounding into consideration, which leads to biased estimates of propensity scores so even after reweighting, the interventional data is not exchangeable with the observational data.
Therefore, the confidence interval constructed by WCP does not have coverage guarantee.
\item Considering interval width, wSCP-DR (Inexact) achieves much better efficiency (narrower interval widths) than Naive for counterfactual outcomes and ITE. 
As wSCP-DR (Exact) expands the confidence interval to gain guaranteed coverage and has slightly smaller interval width than the Naive method. 
WCP has the smallest interval width, however, its confidence intervals cannot contain the ground truth with $0.9$ probability as desired.
In practice, we recommend using wSCP-DR (Inexact) for its enhanced efficiency, if there is no strict requirement on coverage.
\item There is a imbalance of the number of treated and controlled samples in the observational data.
Notice that $c=0.9$ in Eq.~\eqref{eq:DGP} means that the size of controlled group is larger than the size of treated group in observational data.
As a result, compared to Naive method, wTCP-DR has smaller interval width for $Y(0)$, but it has a similar interval width for $Y(1)$, due to the fact that only the number of controlled samples is larger than $m$ while the number of treated samples is at the same scale as $m$. 
This observation verifies the theory of \Cref{thm:main}.
Nevertheless, wTCP-DR's ITE interval is still smaller than Naive.
\end{itemize}
%

%


\noindent\textbf{Impact of Hidden Confounding Strength on Coverage (RQ3).}
Here, we modify the dimensionality of observed covariates $d \in \{1,3,5,10\}$ where larger $d$ means weaker hidden confounding.
Fig.~\ref{fig:impact_dim_x} shows the results with varying hidden confounding strengths.
We make the following observations.
At varying levels of hidden confounding strength, wSCP-DR (Exact) and Naive can maintain the specified level of coverage. In contrast, coverage of wSCP-DR (Inexact) is slightly lower than the specified level.
When hidden confounding is stronger ($d$ is lower), WCP has lower coverage because it ignores hidden confounders and hence its propensity score reweighted conformal prediction does not have guaranteed coverage.
When hidden confounding gets weaker (larger $d$), the coverage of WCP starts to improve, because propensity scores gets closer to the true density ratio that accounts for the distribution shift.

%

\begin{figure*}[ht]
\vspace{-10pt}
    \centering
    \subfloat[Coverage of $Y(0)$]{
\includegraphics[width=0.31\textwidth]{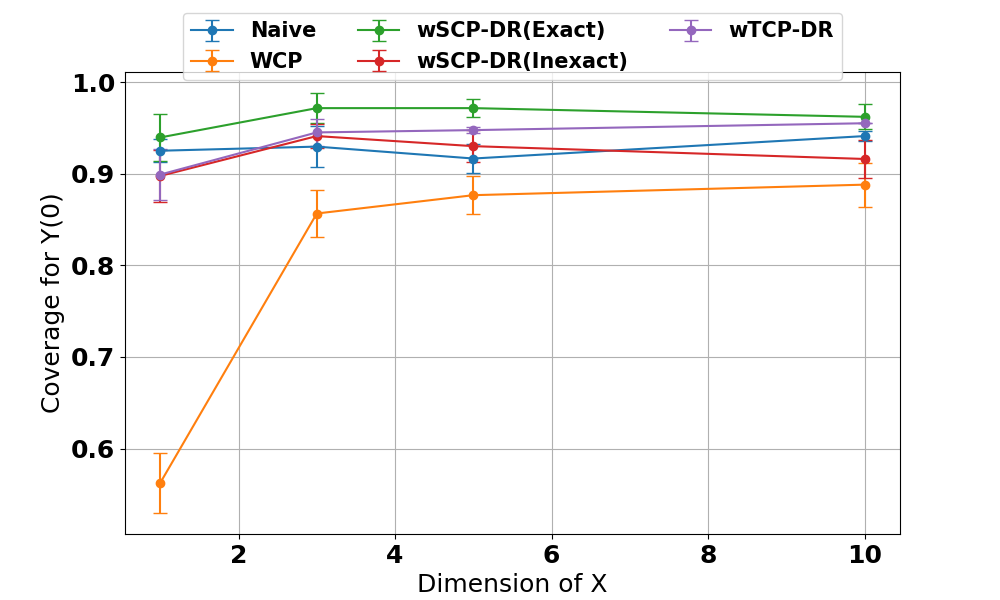}
        \label{fig:cevae_cover_conf_str_y0}
    }
    \hfill
    \subfloat[Coverage of $Y(1)$]{
\includegraphics[width=0.31\textwidth]{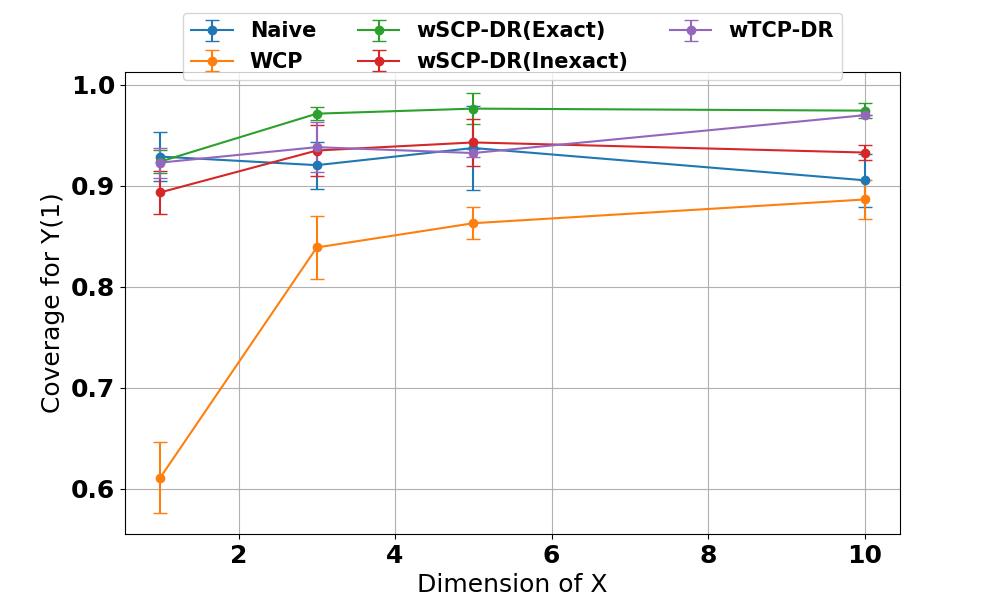}
        \label{fig:cevae_cover_conf_str_ite}
    }
    \hfill
    \subfloat[Coverage of ITE]{
\includegraphics[width=0.31\textwidth]{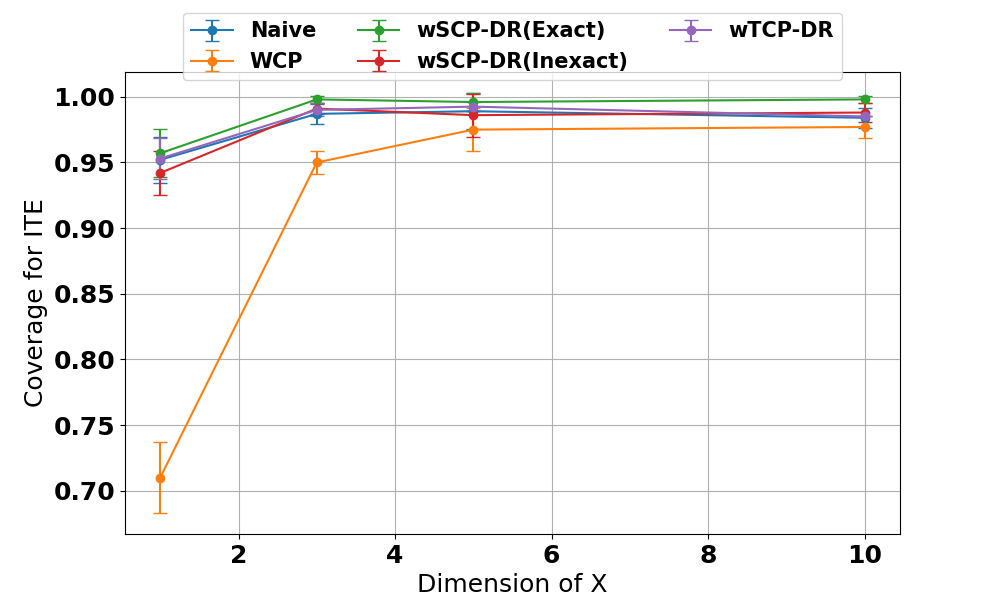}
        \label{fig:cevae_cover_conf_str_ite}
    }
    \vspace{-10pt}
    \caption{Coverage results of counterfactual outcomes and ITE with varying hidden confounding strength. Higher dimensional $X$ carries more information of the hidden confounders, leading to weaker hidden confounding. Their interval width results are in Fig.~\ref{fig:impact_dim_x_interval} of Appendix~\ref{subsec:app_exp_syn}.}
    \label{fig:impact_dim_x}
    \vspace{-10pt}
\end{figure*}

\noindent\textbf{Impact of Interventional Data Size $m$ on Interval Width (RQ4).}
Here, we study the impact of the size of interventional data $m=m_{tr}+m_{cal}$ on interval width, under strong hidden confounding $d=1$.
Fig.~\ref{fig:cevae_m} shows results with different $m$.
The interval width (efficiency) of the Naive method benefit the most from increasing $m$ as its has more training samples and also a larger calibration set for split conformal prediction, which agrees with Eq.~\eqref{eq:naive_coverage}.
%
%
Increasing $m$ has no significant impact on the efficiency of our methods, which agrees with Eq.~\eqref{eq:ours_coverage}.
The reason is that our methods only use interventional data for density ratio estimation, so larger $m$ only improves the quality of estimated density ratios, which does not impact the conformity scores because the scores are computed on the observational data.
For WCP, it does not use interventional data at all, so increasing $m$ also has no impact.
As we discussed before, due to the sample size difference between treatment group and control group, wTCP-DR's efficiency is worse for $Y(1)$ but its interval width for ITE can still be narrower than that of Naive.
\vspace{-10pt}

\begin{figure}[htb!]
    \centering
    \subfloat[Interval width of $Y(0)$ with different $m$]{
\includegraphics[width=0.31\textwidth]{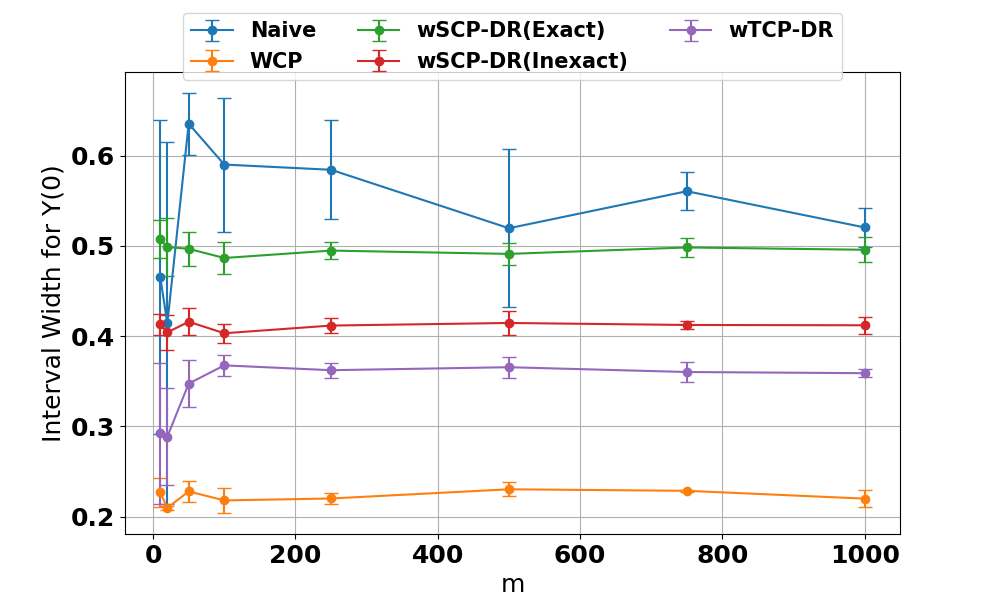}
        \label{fig:cevae_cover_y0_n_int}
    }
    \hfill
    \subfloat[Interval width of $Y(1)$ with different $m$]{
\includegraphics[width=0.31\textwidth]{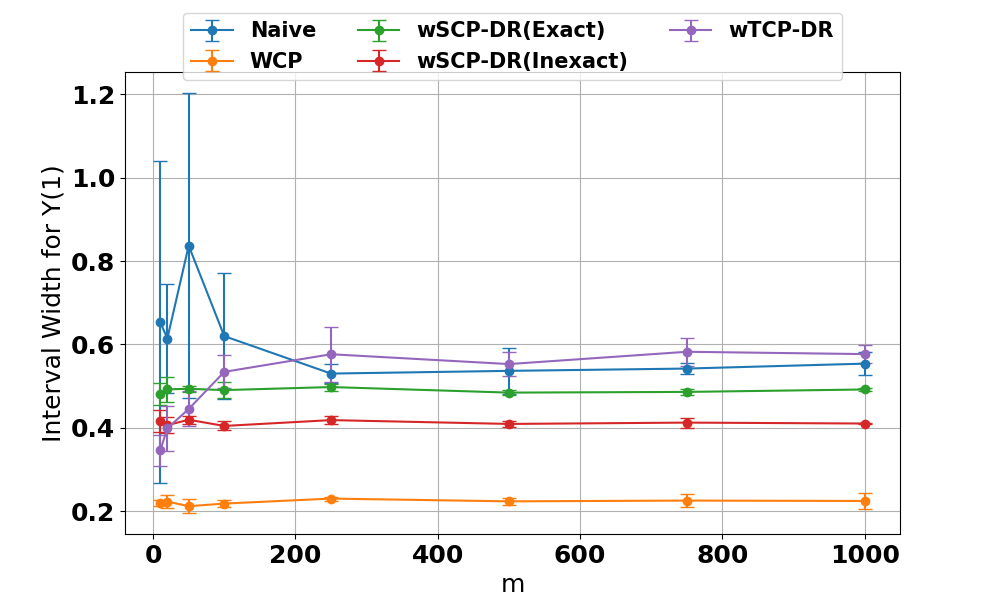}
        \label{fig:cevae_width_y1_n_int}
    }
    \vspace{-10pt}
    \caption{Impact of interventional data size $m$ on efficiency of conformal inference methods. See Appendix~\ref{subsec:app_exp_syn} for coverage results.}
    \label{fig:cevae_m}
\end{figure}



%

\subsection{Counterfactual Outcome Estimation on Real-world Recommendation System Data}

Causal recommendation datasets Yahoo!R3\footnote{https://webscope.sandbox.yahoo.com/} (Yahoo) and Coat\footnote{https://www.cs.cornell.edu/\textasciitilde schnabts/mnar/} can benchmark counterfactual outcome estimation under hidden confounding~\cite{wang2018deconfounded,wang2020causal,zhang2023debiasing}. Note that we use these datasets for counterfactual regression, leaving ranking based evaluation for future work.
Following the formulation of ~\cite{wang2018deconfounded,wang2020causal}, we define each sample as a user-item pair, 
define treatment as whether the item is exposed to the user, and define outcome as the user's rating from 1 to 5. The goal is to predict potential outcome $Y(1)$ for the user-item pairs in the test set $\mathcal{D}_{ts}$ given the learned embeddings of a user-item pair $X$.
Available information include massive observational data from $P_{X,Y|T=1}$ and a small set of interventional data from $P_{X,Y(1)}$.
We run conformal inference on the top of the classic matrix factorization model~\cite{koren2009matrix} trained on $\mathcal{D}^I_{tr}$ for Naive and $\mathcal{D}^O_{tr}$ for other methods.
The size of dataset split can be found in Table~\ref{tab:dataset}.

\noindent\textbf{Methods for Comparison.}
In addition to \textit{wSCP-DR (Inexact and Exact)}, we introduce their variants \textit{wSCP-DR* (Inexact and Exact)} that estimate the density ratio by learned embeddings as $\frac{p^I(x)}{p^O(x)}$. 
This is a favorable setting in practice because randomized controlled trail is costly, whereas randomly assigned users without requiring their outcomes under treatment is much cheaper and easier to implement.
We aim to illustrate our methods can perform well even when there is no access to labeled interventional data.
Here, we do not consider wTCP-DR due to its high computational cost.
For baselines, we use \textit{Naive} and \textit{WCP-NB} -- A variant of WCP which uses interventional data with labels to train a Naive Bayes classifier for estimating propensity scores as in~\cite{schnabel2016recommendations,chen2021autodebias,zhang2023debiasing}.
%
%

\begin{table}[htb]
\centering
\caption{Coverage and interval width results on Yahoo and Coat. Boldface and underlining are used to highlight the top and second-best interval width among the methods with coverage close to 0.9.}
\label{tab:main_res_recsys}
\vspace{-10pt}
\scriptsize
\begin{tabular}{l|c|c|c|c}
\hline    & \multicolumn{2}{c}{Yahoo}  \vline & \multicolumn{2}{c}{Coat} \\
\hline
Method            & {Coverage } \textuparrow  & {Interval Width } \textdownarrow & {Coverage } \textuparrow & {Interval Width} \textdownarrow \\ \hline
wSCP-DR(Inexact)  & $0.892 \pm 0.019$                         & $\mathbf{4.353} \pm 0.019$                              & $0.919 \pm 0.008$                         & $\textbf{3.787} \pm 0.045$                           \\
wSCP-DR(Exact)    & $0.952 \pm 0.001$                         & $5.140 \pm 0.001$                              & $0.959 \pm 0.001$                         & $4.565 \pm 0.228$                           \\
wSCP-DR*(Inexact) & $0.892 \pm 0.020$                         & $\mathbf{4.353} \pm 0.020$                              & $0.919 \pm 0.008$                         & $\underline{3.789} \pm 0.046$                           \\
wSCP-DR*(Exact)   & $0.952 \pm 0.001$                         & $5.140 \pm 0.001$                              & $0.960 \pm 0.001$                         & $4.571 \pm 0.233$                           \\
WCP-NB            & $0.825 \pm 0.002$                         & $4.036 \pm 0.002$                              & $0.912 \pm 0.005$                         & $3.635 \pm 0.040$                           \\
Naive             & $0.899 \pm 0.001$                         & $6.047 \pm 0.001$                              & $0.896 \pm 0.003$                         & $7.725 \pm 0.018$                           \\ \hline
\end{tabular}
\vspace{-10pt}
\end{table}

\noindent\textbf{Comparison Results (RQ1-2).} We fix $m_{tr}=m_{cal}$ for Yahoo and Coat to ensure $n$ larger than $m$ and $m_{ts}$ is large enough (see Table~\ref{tab:dataset}). Studies on $m_{tr}$ and $m_{cal}$ can be found in Appendix~\ref{subsec:app_exp_syn}.
Table~\ref{tab:main_res_recsys} shows results on these two datasets.
Our methods achieve $0.9$ coverage and have significantly smaller intervals than the Naive method. Surprisingly, even when the density ratio is estimated only from the learned embeddings without using interventional labels, our method can still achieve $0.9$ coverage and small intervals.
Therefore, our method has the potential to  completely replace randomized controlled trail with randomized assignation of users when the dimension of the covariate $X$ is higher than the dimension of target $y$, saving huge amounts of resources in practice. 
In contrast, even with interventional data, WCP-NB fails to maintain $0.9$ coverage on the Yahoo dataset because does not take hidden confounding into consideratin.
As expected, Naive has the widest intervals on both datasets while maintaining 0.9 coverage most of the time.




%% file: icml_2024/6_related_work.tex
\section{Related Work}
Estimation of individual treatment effect has been the key for individual decision making in economics~\cite{camerer1995individual}, healthcare~\cite{wendling2018comparing} and education~\cite{zhou2022attendance}.  
Construcing confidence intervals for ITE provides additional information for decision making process to improve its reliability in high-stake situations~\cite{logan2019decision, jesson2020identifying}. 
Previous methods that aim at constructing confidence intervals for the estimation of counterfactual outcomes and individual treatment effects include Bayesian inference~\cite{hill2011bayesian}, bootstrapping~\cite{tu2002bootstrap}, kernel smoothing~\cite{kallus2019interval}, etc.
These methods are known to have aymptotic coverage guarantees (i.e. they require infinite number of samples) and depend on the specific choice of regression models.

Recently, conformal prediction~\cite{tibshirani2019conformal, lei2018distribution} becomes increasingly popular because it has marginal coverage guarantee with finite number of samples and it is also agnostic to the regression model used.
\cite{lei2021conformal} has proposed to use weighted conformal prediction to construct intervals for counterfactuals and ITE, and~\cite{alaa2023conformal} also proposes to use conformal prediction along with meta-learners to construct intervals for ITE. 
However, both \cite{lei2021conformal, yin2022conformal} require strong ignorability assumption and completely ignores the existence of confounding variables, which is unverifiable and unrealistic in practice.
Recently, \cite{jin2023sensitivity} conducts sensitivity analysis of conformal prediction for ITE under hidden confounding, but their method assumes marginal selection condition, another unverifiable assumption in practice.

%% file: icml_2024/7_conclusion.tex
\vspace{-5pt}
\section{Conclusion}
In this paper, we propose a novel algorithm WTCP-DR that provides confidence intervals for predicting counterfactual outcomes and individual treatment effects with guaranteed marginal coverage, even under hidden confounding.
Our theory explicitly demonstrates the conditions under which wTCP-DR is strictly advantageous to the naive method that only uses interventional data.
We also propose a two stage variant called wSCP-DR with the same guarantee at a lower computational cost than wTCP-DR.
We demonstrate that wTCP-DR and wSCP-DR achieve superior performances against state-of-the-art baselines in terms of both coverage and efficiency across synthetic and real-world datasets.

%% file: icml_2024/appendix/main.tex
\begin{appendices}

\crefalias{section}{appendix}
\crefalias{subsection}{appendix}
\crefalias{subsubsection}{appendix}

\setcounter{equation}{0}
\renewcommand{\theequation}{\thesection.\arabic{equation}}

\onecolumn



\input{icml_2024/appendix/2_proof}

\input{icml_2024/appendix/4_lemma}
\input{icml_2024/appendix/3_experiments}

\end{appendices}

%% file: icml_2024/appendix/2_proof.tex
\section{Proof of Theorem 1}\label{appsec:proof_thm1}
Recall that we have access to data $\left(x_i, y_i \right)_{i=1}^{n+m}$, where the first $n$ data are drawn from $p^O(x,y)$ and the last $m$ data are drawn from $p^I(x,y)$.
Our target is to prove that with high probability, the width of the interval $C_{\text{wTCP-DR}}$ constructed from the naive method of \Cref{alg:wtcp_dr} is smaller than the width of the interval $C_{\text{naive}}$ constructed from \Cref{alg:naive}.
We use $x_{n+m+1}$ to denote the test position drawn from the marginal distribution $p(x)$, and we use $\overline{y}$ to denote a pre-selected value from domain $\calY$. 
For the naive method, the interval width is determined by the offset $q_{\widehat{F}_m} = \operatorname{Quantile} \left( 1-\alpha; \frac{1}{m} \sum_{i=1}^{m} \delta_{s_i^{\text{naive}}} \right)$ with $s_i^{\text{naive}}$ being the conformity scores for the naive method.
For our method wTCP-DR, the interval width is determined by the offset $q_{\widehat{F}_n} = \operatorname{Quantile} \left( 1-\alpha; \sum_{i=1}^{n} \hat{p}_i \delta_{s_i^{\text{wTCP-DR}}} + \hat{p}_{n+m+1} \delta_{s_{n+m+1}^{\text{wTCP-DR}}} \right)$ with $s_i^{\text{wTCP-DR}}$ being the conformity scores for wTCP-DR.
In order to prove that the width of $C_{\text{wTCP-DR}}$ is smaller than the width of $C_{\text{naive}}$ is equivalent to prove that
\begin{align}\label{appeq:quantile_inequality}
    \operatorname{Quantile} \left( 1-\alpha; \sum_{i=1}^{n} \hat{p}_i \delta_{s_i^{\text{wTCP-DR}}} + \hat{p}_{n+m+1} \delta_{s_{n+m+1}^{\text{wTCP-DR}}} \right) \leq \operatorname{Quantile} \left( 1-\alpha; \frac{1}{m} \sum_{i=1}^{m} \delta_{s_i^{\text{naive}}} \right)
\end{align}

First, we list all the assumptions required for the proof:
\begin{enumerate}[itemsep=0.1pt,topsep=0pt,leftmargin=*]
\item [A1] Additive Gaussian noise.
\begin{align*}
    y^O \sim \calN({\theta^O}^\top \varphi(x^O), \sigma^2), \quad y^I \sim \calN({\theta^I}^\top \varphi(x^I), \sigma^2)
\end{align*}
\item [A2] Covariates are Gaussianly distributed
\begin{align*}
    \varphi(x^O) \sim \calN(0, \Sigma^O), \quad \varphi(x^I) \sim \calN(0, \Sigma^I)
\end{align*}
\item [A3] Bounded squared difference between oracle density ratio $r(x,y)$ and estimated density ratio $\hat{r}(x,y)$.
\begin{align*}
    \E_{p^O(x,y)} \left( r(x,y) - \hat{r}(x, y)\right)^2 < \infty
\end{align*}
\item [A4] The approximation error of density ratio is upper bounded by $(1 - \alpha) / \alpha$.
\begin{align*}
    \Delta_r = \E_{p^O(x,y)} |r(x,y) - \hat{r}(x, y)| < \frac{1 - \alpha}{\alpha}
\end{align*}
\item [A5] Bounded $\chi^2$ divergence between $p^I(x,y)$ and $p^O(x,y)$.
\begin{align*}
    \chi^2(p^I\| p^O) = \int \left( \frac{p^I(x,y)}{p^O(x,y)} - 1\right)^2 p^O(x,y) dxdy < \infty
\end{align*}
\end{enumerate}

The oracle density ratio is denoted $r(x,y)=p^I(x,y) / p^O(x,y)$ and the estimated density ratio is denoted $\hat{r}(x,y)$.
We know from \eqref{eq:easy_estimated_normalized_weights} that the normalized weights for wTCP-DR are
\begin{align*}
    \hat{p_i} = \frac{ \hat{r}(x_i, y_i) }{\sum_{j=1}^{n} \hat{r}(x_j, y_j) +  \hat{r} (x_{n+m+1}, \overline{y}) } \quad \text{for} \quad i = 1, \cdots, n \quad \quad 
    \hat{p}_{n+m+1} = \frac{ \hat{r}(x_{n+m+1}, \overline{y} ) }{\sum_{j=1}^{n} \hat{r}(x_j, y_j) + \hat{r}(x_{n+m+1}, \overline{y} ) }
\end{align*}

The proof will be divided into three steps.
\begin{enumerate}[itemsep=0.1pt,topsep=0pt,leftmargin=50pt]
\item [Step one:] For wTCP-DR, with probability at least $1 - \delta_1$, $\beta = \alpha - (1 - \alpha) \frac{\hat{p}_{n+m+1}}{1- \hat{p}_{n+m+1}}$ is positive and hence,
\begin{align}\label{appeq:one}
    \operatorname{Quantile} \left( 1-\alpha; \sum_{i=1}^{n} \hat{p}_i \delta_{s_i^{\text{wTCP-DR}}} + \hat{p}_{n+m+1} \delta_{s_{n+m+1}^{\text{wTCP-DR}}} \right) = \operatorname{Quantile} \left( 1-\beta; \sum_{i=1}^{n} \hat{p}_i \delta_{s_i^{\text{wTCP-DR}}} \right)
    \end{align}
\customlabel{step_one}{step one} 
\item [Step two:] Under ordinary least squares (OLS) as the regression model, $s_i^{\text{naive}}$ follow half-Gaussian distribution: $s_1^{\text{naive}}, 
\cdots, s_m^{\text{naive}} \overset{\text{i.i.d}}{\sim} \left| \calN \Big(0, \left( 1 + \frac{d}{m-(d+1)} \right) \sigma^2 \Big) \right|$.
And given i.i.d $v_1, \cdots, v_n \overset{\text{i.i.d}}{\sim} \left| \calN (0, \sigma^2 ) \right|$, with probability at least $1 - \delta_2$,
\begin{align}\label{appeq:two}
    \left| \operatorname{Quantile} \left( 1-\beta; \sum_{i=1}^{n} \hat{p}_i \delta_{s_i^{\text{wTCP-DR}}} \right) - \operatorname{Quantile} \left( 1-\beta; \sum_{i=1}^{n} \hat{p}_i \delta_{v_i} \right) \right| \leq 2 \sigma \sqrt{\frac{\log n}{n}} 
\end{align}
\customlabel{step_two}{step two} 
\item [Step three:] For $s_1^{\text{naive}}, \cdots, s_m^{\text{naive}} \overset{\text{i.i.d}}{\sim} \left| \calN \Big(0, \left( 1 + \frac{d}{m-(d+1)} \right) \sigma^2 \Big) \right|$, and for $v_1, \cdots, v_n \overset{\text{i.i.d}}{\sim} \left| \calN \left(0, \sigma^2 \right) \right|$, we prove that with probability at least $1 - \delta_3$, 
\begin{align}\label{appeq:three}
    \operatorname{Quantile} \left( 1-\beta; \sum_{i=1}^{n} \hat{p}_i \delta_{v_i} \right) \leq \operatorname{Quantile} \left( 1-\alpha; \sum_{i=1}^{m} \frac{1}{m} \delta_{s_i} \right)
\end{align}
\customlabel{step_three}{step three} 
\end{enumerate}
Combining \eqref{appeq:one}, \eqref{appeq:two} and \eqref{appeq:three}, with probability at least $1 - \delta_1 - \delta_2 - \delta_3 - \delta_4$,
\begin{align}
    \operatorname{Quantile} \left( 1-\alpha; \sum_{i=1}^{n} \hat{p}_i \delta_{s_i^{\text{wTCP-DR}}} + \hat{p}_{n+m+1} \delta_{s_{n+m+1}^{\text{wTCP-DR}}} \right) \leq \operatorname{Quantile} \left( 1-\alpha; \frac{1}{m} \sum_{i=1}^{m} \delta_{s_i^{\text{naive}}} \right) + 2 \sigma \sqrt{\frac{\log n}{n}} 
\end{align}
with $\delta_1, \delta_2, \delta_3, \delta_4$ being
\begin{align*}
        &\delta_1 = \left( \frac{2}{n} \frac{1 - \alpha - \frac{\Delta_r}{\Delta_r + 1}}{\alpha + \frac{\Delta_r}{\Delta_r + 1}} \frac{p^O(x)}{ p^I(x) } \right)^{ 4 \sigma^2 \sqrt{\frac{C_1}{C_2}} } , \quad \quad \delta_2 = \frac{2}{n} \\
        &\delta_3 = \exp \left(-\frac{1}{2} L_{1-\alpha}^2 \left(\operatorname{erf}^{-1}(1-\alpha)\right)^2 \frac{(d-1)^2 }{m-1} \right) , \quad \quad \delta_4 = \exp \left(- C_\alpha^2 \frac{ n_{\text{eff}} }{(m-d)^2} \right)
\end{align*}
So we have proved \eqref{appeq:quantile_inequality} and hence proved that the width of $C_{\text{wTCP-DR}}$ is smaller than the width of $C_{\text{naive}}$ up to $\calO\left(\sqrt{\frac{\log n}{n}} \right)$.
Next, we are going to show the proofs for step one, step two and step three respectively.

\begin{proof}[Step one]

In order to prove that $q_{\widehat{F}_n}$ will fall in the conformity scores of the observational data, it is equivalent to prove that $\hat{p}_{n+m+1} \leq \alpha$.
Notice that the difference between the oracle normalized weight $p_{n+m+1} $ and the estimated normalized weight $\hat{p}_{n+m+1}$ is
\begin{align*}
    &\big| p_{n+m+1} - \hat{p}_{n+m+1} \big| = \left| \frac{ r (x_{n+m+1}, \overline{y}) }{\sum_{j=1}^{n} r (x_j, y_j) +  r (x_{n+m+1}, \overline{y}) } - \frac{ \hat{r}(x_{n+m+1}, \overline{y}) }{ \sum_{j=1}^{n} \hat{r}(x_j, y_j) +  \hat{r} (x_{n+m+1}, \overline{y}) } \right| \\
    &= \left| \frac{ r(x_{n+m+1}, \overline{y}) \sum_{j=1}^{n} \hat{r}(x_j, y_j) - \hat{r}(x_{n+m+1}, \overline{y}) \sum_{j=1}^{n} r(x_j, y_j) }{ \left( \sum_{j=1}^{n} r (x_j, y_j) +  r (x_{n+m+1}, \overline{y}) \right) \left( \sum_{j=1}^{n} \hat{r} (x_j, y_j) +  \hat{r} (x_{n+m+1}, \overline{y}) \right) } \right| \\
    &= \left| \frac{ r (x_{n+m+1}, \overline{y}) \left( \sum_{j=1}^{n} \hat{r}(x_j, y_j) - \sum_{j=1}^{n} r(x_j, y_j) \right) + ( r(x_{n+m+1}, \overline{y}) - \hat{r}(x_{n+m+1}, \overline{y})) \sum_{j=1}^{n} r (x_j, y_j) }{ \left( \sum_{j=1}^{n} r (x_j, y_j) +  r (x_{n+m+1}, \overline{y}) \right) \left( \sum_{j=1}^{n} \hat{r} (x_j, y_j) +  \hat{r} (x_{n+m+1}, \overline{y}) \right)}  \right| \\
    &\leq \left| \frac{ \sum_{j=1}^{n} \hat{r}(x_j, y_j) - \sum_{j=1}^{n} r(x_j, y_j) + r(x_{n+m+1}, \overline{y}) - \hat{r}(x_{n+m+1}, \overline{y}) }{ \sum_{j=1}^{n} \hat{r} (x_j, y_j) +  \hat{r} (x_{n+m+1}, \overline{y}) } \right| \\
    &\leq \left| \frac{ \left( \sum_{j=1}^{n} \hat{r}(x_j, y_j) - \sum_{j=1}^{n} r(x_j, y_j) \right) + r(x_{n+m+1}, \overline{y}) - \hat{r}(x_{n+m+1}, \overline{y}) }{ \left( \sum_{j=1}^{n} \hat{r} (x_j, y_j) - \sum_{j=1}^{n} r(x_j, y_j) \right) + \sum_{j=1}^{n} r(x_j, y_j)  }  \right| \\
    &\leq  \frac{  \sum_{j=1}^{n} \left| \hat{r}(x_j, y_j) - r(x_j, y_j) \right| + \left| r(x_{n+m+1}, \overline{y}) - \hat{r}(x_{n+m+1}, \overline{y})  \right| }{ \sum_{j=1}^{n} \left| \hat{r}(x_j, y_j) - r(x_j, y_j) \right| + \sum_{j=1}^{n} r(x_j, y_j)  }  \\
    &= \frac{  \frac{1}{n} \sum_{j=1}^{n} \left| \hat{r}(x_j, y_j) - r(x_j, y_j) \right| + \frac{1}{n}  \left| r(x_{n+m+1}, \overline{y}) - \hat{r}(x_{n+m+1}, \overline{y})  \right| }{ \frac{1}{n} \sum_{j=1}^{n} \left| \hat{r}(x_j, y_j) - r(x_j, y_j) \right| + \frac{1}{n} \sum_{j=1}^{n} r(x_j, y_j)  }  \\
    &= \frac{\Delta_r}{\Delta_r + 1} + \calO_p(n^{-1/2})
\end{align*}
The second last equality is by noticing that $\frac{1}{n} \sum_{j=1}^{n} \left| \hat{r}(x_j, y_j) - r(x_j, y_j) \right| $ is sample approximation of $\Delta_r = \E_{p^O(x,y)} |r(x,y) - \hat{r}(x, y)| $, and $\frac{1}{n} \sum_{j=1}^{n} r(x_j, y_j) = \frac{1}{n} \sum_{j=1}^{n} 
\frac{p^I(x_j, y_j)}{p^O(x_j, y_j)}$ is sample approximation of $\int \frac{p^I(x, y)}{p^O(x, y)} p^O(x, y) d(x,y) = \int p^I(x, y) d(x,y) = 1$, so central limit theorem tells us that $\frac{1}{n} \sum_{j=1}^{n} \left| \hat{r}(x_j, y_j) - r(x_j, y_j) \right| = \Delta_r + \calO_p(n^{-1/2})$ and $ \frac{1}{n} \sum_{j=1}^{n} r(x_j, y_j) = 1  + \calO_p(n^{-1/2})$.
When $\Delta_r = 0$, i.e the estimated density ratio $\hat{r}$ recover the oracle density ratio $r$, $p_{n+m+1} - \hat{p}_{n+m+1} = 0$. Since convergence in probability implies convergence in distribution, we have 
\begin{align}
    \Pb(\hat{p}_{n+m+1} \leq \alpha) &\geq \Pb \left(p_{n+m+1} \leq \alpha + \frac{\Delta_r}{\Delta_r + 1} \right) + \calO(n^{-1/2}) \nonumber \\
    &= \Pb \left( r (x_{n+m+1}, \overline{y}) \leq \frac{\alpha + \frac{\Delta_r}{\Delta_r + 1}}{1 - \alpha - \frac{\Delta_r}{\Delta_r + 1}}  \sum \limits_{j=1}^{n} r(x_j, y_j) \right) + \calO(n^{-1/2}) 
    \nonumber \\
    &\geq 1 - \left( \frac{2}{n} \frac{1 - \alpha - \frac{\Delta_r}{\Delta_r + 1}}{\alpha + \frac{\Delta_r}{\Delta_r + 1}} \frac{p^O(x)}{ p^I(x) } \right)^{ 4 \sigma^2 \sqrt{\frac{C_1}{C_2}} } \label{appeq:p_small_than_alpha}
\end{align}

The second equality holds when $1 - \alpha - \frac{\Delta_r}{\Delta_r + 1} > 0$, and since typically $\alpha$ takes small values like $0.1$. 
The final inequality is using \Cref{prop:prob_n+m+1}.

Therefore, denoting $\beta=1 - (1-\alpha) \left(1 + \frac{\hat{p}_{n+m+1}}{1-\hat{p}_{n+m+1}} \right) = \alpha - (1 - \alpha) \frac{\hat{p}_{n+m+1}}{1- \hat{p}_{n+m+1}} \geq 0$, with probability at least $1 - \left( \frac{2}{n}  \frac{1 - \alpha - \frac{\Delta_r}{\Delta_r + 1}}{\alpha + \frac{\Delta_r}{\Delta_r + 1}} \frac{p^O(x)}{ p^I(x) } \right)^{ 4 \sigma^2 \sqrt{\frac{C_1}{C_2}} }$,
\begin{align*}
    \operatorname{Quantile} \left( 1-\alpha; \sum_{i=1}^{n} \hat{p}_i \delta_{s_i} + \hat{p}_{n+m+1} \delta_{s_{n+m+1}} \right) = \operatorname{Quantile} \left( 1 - \beta; \sum_{i=1}^{n} \hat{p}_i \delta_{s_i} \right)
\end{align*}
Up till this point, \ref{step_one} has finished.
\end{proof}

\begin{proof}[Step two]

First, we consider the conformity scores $s_1^{\text{naive}}, \cdots, s_m^{\text{naive}}$ of the naive approach in \Cref{alg:naive}.
Recall that $m/2$ interventional data $(x_{n+1}, y_{n+1}), \cdots, (x_{n+m/2}, y_{n+m/2})$ are used for training the regression model $\hat{f}^\text{naive}$, and $m/2$ interventional data $(x_{n+m/2+1}, y_{n+m/2+1}), $ $\cdots, (x_{n+m}, y_{n+m})$ are used for constructing confidence interval.
For $i=n+m/2+1, \cdots, n+m$, we know from \Cref{prop:error} that $y_i - \hat{f}^\text{naive}(x_i)$ follows Gaussian distribution with mean $0$ and variance $\frac{d}{m / 2 - (d+1)} \sigma^2$, so the conformity score $s_i^{\text{naive}} = |y_i - \hat{f}^\text{naive}(x_i)|$ follows half-Gaussian distribution. 

Next, we consider the conformity scores $s_1^{\text{wTCP-DR}}, \cdots, s_m^{\text{wTCP-DR}}$ of wTCP-DR in \Cref{alg:wtcp_dr}.
Recall that the observational samples are $(x_1, y_1), \cdots, (x_n, y_n)$, the test covariate is $x_{n+m+1}$ and $\overline{y}$ are selected from a predefined domain $\calY$.
After constructing an augmented dataset $(x_1, y_1), \cdots, (x_n, y_n), (x_{n+m+1}, \overline{y})$ and training a regression model $\hat{f}^\text{wTCP-DR}$ on the dataset, the conformity score $s_i^{\text{wTCP-DR}}$ is the absolute difference $s_i^{\text{wTCP-DR}} = \big|y_i - \hat{f}^{\text{wTCP-DR}}(x_i)\big|$.

Denote $\bar{f}^{\text{wTCP-DR}}$ as the OLS regressor obtained from data $(x_1, y_1), \cdots, (x_n, y_n)$ without $(x_{n+m+1}, \overline{y})$.
From \Cref{prop:perturb_one_stability}, we know that with probability at least $1 - \frac{1}{n}$,
\begin{align*}
    \Big| \big|y_i - \bar{f}^{\text{wTCP-DR}}(x_i)\big| - \big|y_i - \hat{f}^{\text{wTCP-DR}}(x_i) \big| \Big| \leq \sigma \sqrt{ \frac{\log n}{n}}
\end{align*}
And from \Cref{prop:error}, we know that with probability at least $1 - \frac{1}{n}$,
\begin{align*}
    \Big| \big|y_i - \bar{f}^{\text{wTCP-DR}}(x_i)\big| - \big|y_i - f(x_i) \big| \Big| \leq \sigma \sqrt{ \frac{\log n}{n}} 
\end{align*}
where $f(x) = {\theta^I}^\top \varphi(x) $ is the ground truth. From assumption we know that $v_i = \big|y_i - f(x_i) \big|$ follows half-Gaussian distribution. Combining the above two inequalities, we know that with probability at least $1 - \frac{2}{n}$, $\left|s_i^{\text{wTCP-DR}} - v_i \right| \leq 2 \sigma \sqrt{ \frac{\log n}{n}}  $, and consequently
\begin{align*}
    \left| \operatorname{Quantile} \left( 1-\beta; \sum_{i=1}^{n} \hat{p}_i \delta_{s_i^{\text{wTCP-DR}}} \right) - \operatorname{Quantile} \left( 1-\beta; \sum_{i=1}^{n} \hat{p}_i \delta_{v_i} \right) \right| \leq 2 \sigma \sqrt{ \frac{\log n}{n}}
\end{align*}
Up till this point, \ref{step_two} has finished.
\end{proof}

\begin{proof}[Step three]
First, for $\operatorname{Quantile} \left( 1-\alpha; \sum_{i=1}^{m} \frac{1}{m} \delta_{s_i} \right)$, consider the probability
\begin{align*}
    & \Pb \left(\operatorname{Quantile} \left( 1-\alpha; \sum_{i=1}^{m} \frac{1}{m} \delta_{s_i} \right) \leq \sqrt{2}\sigma \operatorname{erf}^{-1}(1 - \alpha) \sqrt{ \frac{m-d}{m-d-1}} \right) \\ 
    &= \sum_{k= \lceil m(1-\alpha) \rceil} C_m^k F \left( \sqrt{2}\sigma \operatorname{erf}^{-1}(1 - \alpha) \sqrt{ \frac{m-d}{m-d-1}} \right)^k \left(1 - F \left( \sqrt{2}\sigma \operatorname{erf}^{-1}(1 - \alpha) \sqrt{\frac{m-d}{m-d-1}} \right) \right)^{m-k} \\
    &= \sum_{k= \lceil m(1-\alpha) \rceil} C_m^k \left( \operatorname{erf}\left( \operatorname{erf}^{-1}(1 - \alpha) \sqrt{\frac{m-d}{m-1 } } \right) \right)^k \left(1 - \operatorname{erf}\left( \operatorname{erf}^{-1}(1 - \alpha) \sqrt{\frac{m-d}{m-1}} \right) \right)^{m-k}
\end{align*}
where $F$ is the CDF for half-Gaussian random variable $\left| \calN \Big(0, \left( 1 + \frac{d}{m-(d+1)} \right) \sigma^2 \Big) \right|$,  $\operatorname{erf}$ is the error function and $C_m^k$ is the combinatorial number.
The second equality is using \Cref{lem:order_statistic} the CDF for order statistics and the third equality is using \Cref{lem:half_normal} the CDF for half-Gaussian random variable.

Notice that $\operatorname{erf}(\operatorname{erf}^{-1}(1 - \alpha) - x) \leq (1-\alpha) - L_{1 - \alpha} x $ holds for any positive $x$ with $L_{1 - \alpha} $ being the derivative of $\operatorname{erf}$ at $\operatorname{erf}^{-1}(1 - \alpha)$.
\begin{align}
    \operatorname{erf}\left( \operatorname{erf}^{-1}(1 - \alpha) \sqrt{\frac{m-d}{m-1}} \right) 
    &= \operatorname{erf}\left( \operatorname{erf}^{-1}(1 - \alpha) - \operatorname{erf}^{-1}(1 - \alpha) \left( 1 - \sqrt{\frac{m-d}{m-1}} \right) \right) \nonumber \\
    &\leq (1-\alpha) - L_{1 - \alpha} \operatorname{erf}^{-1}(1-\alpha) \left( 1 - \sqrt{\frac{m-d}{m-1}} \right) \nonumber \\
    &\leq (1-\alpha) - L_{1 - \alpha} \operatorname{erf}^{-1}(1-\alpha) \frac{d-1}{2(m-1) } \label{appeq:erf_temp}
\end{align}
So, we have 
\begin{align}
    &\Pb \left(\operatorname{Quantile} \left( 1-\alpha; \sum_{i=1}^{m} \frac{1}{m} \delta_{s_i} \right) \leq \sqrt{2}\sigma \operatorname{erf}^{-1}(1 - \alpha) \sqrt{ \frac{m-d}{m-d-1} } \right) \nonumber \\
    &\leq \sum_{k= \lceil m(1-\alpha) \rceil} C_m^k \left((1-\alpha) - L_{1 - \alpha} \operatorname{erf}^{-1}(1-\alpha) \frac{d-1}{2(m-1) } \right)^k \left(1 - \left( (1-\alpha) - L_{1 - \alpha} \operatorname{erf}^{-1}(1-\alpha) \frac{d-1}{2(m-1) } \right) \right)^{m-k} \nonumber \\
    &\leq \exp \left(-2 m  \left( (1-\alpha) - \left( (1-\alpha) - L_{1 - \alpha} \operatorname{erf}^{-1}(1-\alpha) \frac{d-1}{2(m-1) } \right) \right)^2 \right) \nonumber \\
    &\leq \exp \left(-\frac{1}{2} L_{1-\alpha}^2 \left(\operatorname{erf}^{-1}(1-\alpha)\right)^2 \frac{(d-1)^2 }{m-1} \right) \label{appeq:event_b}
\end{align}
The first inequality is using \eqref{appeq:erf_temp} and the fact that the mapping $x \to \sum_{k= \lceil m(1-\alpha) \rceil}^m C_m^k x^m (1-x)^{m-k}$ is monotonically increasing with $0 \leq x \leq 1$ and the second inequality is using \Cref{lem:hoeffding_bernouli}.

Next, denoting the effective sample size $n_{\text{eff}}= 1 / \sum_{i=1}^n \hat{p}_i^2$, the central limit theorem of weighted empirical quantiles \Cref{prop:weighted_quantile_clt} shows 
\begin{align}
    \sqrt{ n_{\text{eff}} } \left( \operatorname{Quantile} \left( 1-\beta; \sum_{i=1}^{n} \hat{p}_i \delta_{v_i} \right) - \sqrt{2} \sigma 
    \operatorname{erf}^{-1}(1-\beta) \right) \stackrel{d}{\longrightarrow} \calN \left(0 , \frac{\beta (1 - \beta)}{f_V \Big(\sqrt{2} \sigma \operatorname{erf}^{-1}(1-\beta) \Big)^2} \right)
\end{align}
where $f_V$ is the probability density function for half-Gaussian random variable $| \calN(0, \sigma^2)|$, so
\begin{align}\label{appeq:fZ}
    f_V \left(\sqrt{2} \sigma \operatorname{erf}^{-1}(1-\beta) \right) = \frac{1}{\sigma} \underbrace{\sqrt{\frac{2}{\pi}} \exp\left(-\operatorname{erf}^{-1}(1-\beta)^2 \right)}_{C_f}
\end{align}
So, we have
\begin{align}
    &\Pb \left( \operatorname{Quantile} \left( 1-\beta; \sum_{i=1}^{n} \hat{p}_i \delta_{v_i} \right) \leq \sqrt{2}\sigma \operatorname{erf}^{-1}(1 - \alpha) \sqrt{ \frac{m-d}{m-d-1} } \right) \nonumber \\
    &= 1 - \frac{1}{\sqrt{2\pi}} \Phi \left(\frac{ \sqrt{2}\sigma \operatorname{erf}^{-1}(1 - \alpha) \sqrt{ \frac{m-d}{m-d-1}} - \sqrt{2} \sigma 
    \operatorname{erf}^{-1}(1-\beta)}{ \frac{1}{\sqrt{n_{\text{eff}} }} \frac{ \sqrt{ \beta (1 - \beta)} }{f_V \left(\sqrt{2} \sigma \operatorname{erf}^{-1}(1-\beta) \right) } } \right) \nonumber \\
    &= 1 - \frac{1}{\sqrt{2\pi}} \Phi \left( 
    \frac{\sqrt{2} \sigma \operatorname{erf}^{-1}(1 - \alpha)  \frac{\sqrt{m-d} -\sqrt{m-d-1} }{ \sqrt{ m-d-1}} - \sqrt{2}\sigma \left( \operatorname{erf}^{-1}(1 - \beta) - \operatorname{erf}^{-1}(1 - \alpha) \right) }{ \frac{1}{\sqrt{n_{\text{eff}} }} \frac{ \sqrt{ \beta (1-\beta)} \sigma}{C_f} } \right) \nonumber \\
    &\geq 1 - \frac{1}{\sqrt{2\pi}} \Phi \left( 
    \frac{\sqrt{2} \sigma \operatorname{erf}^{-1}(1 - \alpha)  \frac{1}{2( m-d) } - \sqrt{2}\sigma \left( \operatorname{erf}^{-1}(1 - \beta) - \operatorname{erf}^{-1}(1 - \alpha) \right) }{ \frac{1}{\sqrt{n_{\text{eff}} }} \frac{ \sqrt{ \beta (1-\beta)} \sigma}{C_f} } \right) \nonumber  \\ 
    &\asymp 1 - \frac{1}{\sqrt{2\pi}} \Phi \left( \frac{1}{\sqrt{2}} \underbrace{ \frac{C_f \operatorname{erf}^{-1}(1 - \alpha)}{\sqrt{ \beta (1-\beta)} } }_{C_\alpha} \frac{ \sqrt{n_{\text{eff}}} }{m-d} \right) \nonumber \\
    &\gtrsim 1 - \frac{1}{\sqrt{2\pi}} \exp \left(- C_\alpha^2 \frac{ n_{\text{eff}} }{(m-d)^2} \right) \label{appeq:event_a}
\end{align}
The first equality is using the definition of  $\Phi(x) = \int_x^\infty \exp(-\frac{1}{2} t^2) dt$, the second equality is using \eqref{appeq:fZ}, the fourth equality is using the fact that $(1 - \beta) - (1 - \alpha) = (1-\alpha) \frac{\hat{p}_{n+m+1}}{1-\hat{p}_{n+m+1}}$ and $\operatorname{erf}^{-1}$ has bounded Lipschitz constant at $1-\alpha$ and the last equality is using \Cref{lem:erf}.

Denote event $\calA = \{ \operatorname{Quantile} \left( 1-\beta; \sum_{i=1}^{n} \hat{p}_i \delta_{v_i} \right) \leq \sqrt{2}\sigma \operatorname{erf}^{-1}(1 - \alpha) \frac{m-d}{m-d-1} \}$, and the event $\calB = \{ \operatorname{Quantile} \left( 1-\alpha; \sum_{i=1}^{m} \frac{1}{m} \delta_{s_i} \right) \leq \sqrt{2}\sigma \operatorname{erf}^{-1}(1 - \alpha) \frac{m-d}{m-d-1} \}$. 
From the above two inequalities \eqref{appeq:event_a} and \eqref{appeq:event_b}, we know that, $\Pb(\calA) \geq 1 - \exp \left( -C_\alpha^2 \frac{n}{(m-d-1)^2} \right)$ and $\Pb(\calB) \leq \exp \left(-2 \gamma^2 \operatorname{erf}^{-1}(1-\alpha)^2 \frac{(d-1)^2 }{m} \right)$.
Using the inequality that $\Pb(\calA \cap B^\complement) \geq \Pb(\calA) - \Pb(\calB)$, we finally have
\begin{align*}
\Pb \left( \operatorname{Quantile} \left( 1-\beta; \sum_{i=1}^{n} \hat{p}_i \delta_{v_i} \right) \leq \operatorname{Quantile} \left( 1-\alpha; \sum_{i=1}^{m} \frac{1}{m} \delta_{s_i} \right) \right) \\
    \geq 1 - \exp \left(-\frac{1}{2} L_{1-\alpha}^2 \left(\operatorname{erf}^{-1}(1-\alpha)\right)^2 \frac{(d-1)^2 }{m-1} \right) - \exp \left(- C_\alpha^2 \frac{ n_{\text{eff}} }{(m-d)^2} \right)
\end{align*}
Up till this point, \ref{step_three} has finished.

\end{proof}

\begin{prop}\label{prop:prob_n+m+1}
Given $n$ samples $(x_1, y_1), \cdots, (x_n, y_n) \sim p^O(x,y)=\calN({\theta^O}^\top \varphi(x), \sigma^2)p^O(x)$ and given another sample $(x,y) \sim p^I(x,y)=\calN({\theta^I}^\top \varphi(x), \sigma^2)p^I(x)$, denote the density ratio $r(x,y) = p^I(x,y) / p^O(x,y)$,
then for any $\gamma > 0$, we have
\begin{align}
    \Pb \left( r (x, y) \leq \gamma \sum \limits_{j=1}^{n} r(x_j, y_j) \right) \geq 1 - \left( \frac{2}{n \gamma} \frac{p^O(x)}{ p^I(x) } \right)^{ \frac{4 \sigma^2}{\sqrt{C_1 C_2}} } 
\end{align}
where $C_1 = (\theta^I + \theta^O)^\top \Sigma^I (\theta^I + \theta^O)$ and $C_2 = (\theta^I - \theta^O)^\top \Sigma^I (\theta^I - \theta^O)$.
\end{prop}
\begin{proof}
The density ratio can be factorized as $r(x,y) = \frac{p^I(x, y) }{p^O (x, y) } = \frac{p^I(x)}{p^O(x)} \frac{p^I(y \mid x)}{p^O(y \mid x)}$
where
\begin{align*}
    \frac{p^I(y \mid x)}{p^O(y \mid x)} = \frac{\exp \Big( -\frac{1}{2 \sigma^2} \left(y - {\theta^I}^\top \varphi(x) \right)^2 \Big) }{ \exp \Big( -\frac{1}{2 \sigma^2} \left(y - {\theta^O}^\top \varphi(x) \right)^2 \Big) }
\end{align*}
By denoting the random variable $\xi_1 = (\theta^I + \theta^O)^\top \varphi(x)$ which is Gaussianly distributed with mean $0$ and variance $C_1 = (\theta^I + \theta^O)^\top \Sigma^I (\theta^I + \theta^O)$ and denoting $\xi_2 = (\theta^I - \theta^O)^\top \varphi(x)$ which is also Gaussianly distributed with mean $0$ and variance $C_2 = (\theta^I - \theta^O)^\top \Sigma^I (\theta^I - \theta^O)$, so:
\begin{align}
    \log \left( \frac{p^I(y \mid x)}{p^O(y \mid x)} \right) = -\frac{1}{2 \sigma^2} \Big( 2y - (\theta^I + \theta^O)^\top \varphi(x) \Big) (\theta^I - \theta^O)^\top \varphi(x) = -\frac{1}{2 \sigma^2} ( 2y - \xi_1) \xi_2 \nonumber
\end{align}
Consider the probability $\Pb \left( \log \left( \frac{p^I(y \mid x)}{p^O(y \mid x)} \right) \leq t \right)$ for large positive $t$:
\begin{align}
    \Pb  \left( \log \left( \frac{p^I(y \mid x)}{p^O(y \mid x)} \right) \leq t \right) 
    &\geq \Pb \Big( \frac{1}{2 \sigma^2} |2y - \xi_1| | \xi_2 | \leq t \Big) \nonumber \\
    &\geq \Pb \Big(\bigg\{ \bigcup_{z>0} |2y - \xi_1| \leq \sqrt{2} \sigma z, |\xi_2| \leq \sqrt{2} \sigma t/z \bigg\} \Big) \nonumber \\
    &\geq \Pb \Big(   |2y - \xi_1| \leq \sqrt{2} \sigma \sqrt{t} (C_1 / C_2)^{1/4}, |\xi_2| \leq \sqrt{2} \sigma \sqrt{t} (C_2 / C_1)^{1/4} \} \Big) \nonumber \\
    &\geq  1 - \Pb \Big(   |2y - \xi_1| \geq \sqrt{2} \sigma \sqrt{t} (C_1 / C_2)^{1/4} \Big)  - \Pb \Big( |\xi_2| \geq \sqrt{2} \sigma \sqrt{t} (C_2 / C_1)^{1/4} \Big) \nonumber \\
    &=  1 - \Phi \left( \frac{\sqrt{2}\sigma \sqrt{t} (C_1 / C_2)^{1/4} }{\sqrt{C_1}} \right) - \Phi \left( \frac{\sqrt{2}\sigma \sqrt{t} (C_2 / C_1)^{1/4} }{\sqrt{C_2}} \right) \nonumber \\
    &\gtrsim 1 - \exp \left( -4 \sigma^2 t \frac{1}{\sqrt{C_1 C_2}} \right) - \exp \left( -4 \sigma^2 t \frac{1}{\sqrt{C_1 C_2}} \right) \nonumber \\
    &= 1 - 2 \exp \left( -4 \sigma^2 t \frac{1}{\sqrt{C_1 C_2}} \right) \label{appeq:cdf_for_log_w}
\end{align}
where $\Phi(x) = \int_x^\infty \exp(-\frac{1}{2} t^2) dt$. The third equality is by taking $z=\sqrt{t}(C_1 / C_2)^{1/4}$, the fourth equality is using the fact that $\Pb(\calA \cap \calB) \geq 1 - \Pb(\overline{\calA}) - \Pb(\overline{\calB})$ for any two events $\calA, \calB$, the fifth equality is using the definition of $\Phi$, and the sixth inequality is using \Cref{lem:erf}.

Noticing that $\frac{1}{n} \sum_{j=1}^{n} r(x_j, y_j) = \frac{1}{n} \sum_{j=1}^{n} 
\frac{p^I(x_j, y_j)}{p^O(x_j, y_j)}$ is nothing but a sample approximation of $\int \frac{p^I(x, y)}{p^O(x, y)} p^O(x, y) d(x,y) = \int p^I(x, y) d(x,y) = 1$, so the central limit theorem tells us that $\frac{1}{n} \sum_{j=1}^{n} r(x_j, y_j) = 1 + \calO_\Pb(n^{-1/2})$, and hence $\log \left(\frac{1}{n} \sum_{j=1}^{n} r(x_j, y_j) \right) = \calO_\Pb(n^{-1/2})$.
Finally, we have
\begin{align}
    \Pb \left( r (x, y) \leq \gamma \sum \limits_{j=1}^{n} r(x_j, y_j) \right) 
    &= \Pb \left( \frac{p^I(x)}{ p^O(x) } \frac{p^I(y \mid x)}{p^O(y \mid x)}  \leq \gamma \sum \limits_{j=1}^{n} r(x_j, y_j) \right) \nonumber \\
    &= \Pb \left( \frac{p^I(y \mid x)}{p^O(y \mid x)} \leq  \frac{p^O(x)}{ p^I(x) }  \gamma \sum \limits_{j=1}^{n} r(x_j, y_j) \right) \nonumber \\
    &= \Pb \left( \log \left( \frac{p^I(y \mid x)}{p^O(y \mid x)} \right) \leq \log \left( \frac{p^O(x)}{ p^I(x) } \right) + \log  \gamma  + \log \left( \sum \limits_{j=1}^{n} r(x_j, y_j) \right) \right) \nonumber \\
    &= \Pb \left(\log \left( \frac{p^I(y \mid x)}{p^O(y \mid x)} \right) \leq \log \left(  \frac{p^O(x)}{ p^I(x) }  \right) + \log  \gamma + \log n \right) + \calO(n^{-1/2}) \nonumber \\
    &\geq  1- 2\exp \left( -4 \frac{\sigma^2 \left( \log n + \log  \gamma + \log \left( \frac{p^O(x)}{ p^I(x) } \right) \right) }{\sqrt{C_1 C_2} } \right) \nonumber \\
    &= 1 - \left( \frac{2}{n \gamma} \frac{p^O(x)}{ p^I(x) } \right)^{ \frac{4 \sigma^2}{\sqrt{C_1 C_2}} } \label{appeq:w_prob_bound}
\end{align}

Without loss of generality, it is safe to assume that $C_1 = 1$, and so we have
\begin{align*}
    \Pb \left( r (x, y) \leq \gamma \sum \limits_{j=1}^{n} r(x_j, y_j) \right) &\geq 1 - \left( \frac{2}{n \gamma} \frac{p^O(x)}{ p^I(x) } \right)^{ 4 \sigma^2 \sqrt{\frac{C_1}{C_2}}}
\end{align*}
and the proof is finished.

Notice that the probability $1 - \left( \frac{2}{n \gamma} \frac{p^O(x)}{ p^I(x) } \right)^{ 4 \sigma^2 \sqrt{\frac{C_1}{C_2}} } \to 1$ as $n \to \infty$, however the rate at which the probability goes to $1$ is determined by the exponent $\sqrt{\frac{C_1}{C_2}} = \sqrt{\frac{(\theta^I + \theta^O)^\top \Sigma^I (\theta^I + \theta^O)}{(\theta^I - \theta^O)^\top \Sigma^I (\theta^I - \theta^O)}}$. 
When $\theta^I$ and $\theta^O$ are very close, which means that the distribution shift from $p^I(x,y)$ to $p^O(x,y)$ is also very small, $r(x,y)$ is small and very likely to be smaller than $\gamma \sum_{j=1}^{n} r(x_j, y_j)$. 
In contrast, when $\theta^I$ and $\theta^O$ are very different, which means that the distribution shift from $p^I(x,y)$ to $p^O(x,y)$ is very large, $r(x,y)$ is large so more samples are needed to make $\gamma \sum_{j=1}^{n} r(x_j, y_j)$ larger than $r(x,y)$.
\end{proof}

\begin{prop}\label{prop:error}
Given samples $(x_1, y_1), \cdots, (x_n, y_n)$, with $y_i = \theta^\top \varphi(x_i) + \epsilon_i$ where $\epsilon_i$ are independent Gaussian noise random variables of mean $0$ and variance $\sigma^2$ and covariates $\varphi(x_i) \sim \calN(0, \Sigma)$, the ordinary least squares regression model returns an estimator $\hat{\theta} = (\Phi^\top \Phi)^{-1} \Phi^\top y_{1:n}$.
Then,
\begin{enumerate}[itemsep=0.1pt,topsep=0pt,leftmargin=*]
    \item For a test sample $(x,y)$ drawn from the same distribution as $(x_1, y_1), \cdots, (x_n, y_n)$, the test error $r:=y - \varphi(x)^\top \hat{\theta}$ follows a Gaussian distribution with mean $0$ and variance $\left(1 + \frac{d}{n - d - 1} \right) \sigma^2$.
    \item $\Pb \left( \left| |y_i - \varphi(x)^\top \hat{\theta}| - |y_i - \varphi(x)^\top \theta| \right| \leq \sqrt{ \frac{\log n}{n}} \right) \geq 1 - \frac{1}{n}$.
\end{enumerate}

\end{prop}
\begin{proof}
Plugging in the OLS estimator $\hat{\theta}$ into test error $r$, we have
\begin{align*}
r &:= y - \varphi(x)^\top \hat{\theta} =  \varphi(x)^\top \theta + \epsilon - \varphi(x_i)^\top \hat{\theta}
= \epsilon + \varphi(x)^\top \theta - \varphi(x)^\top ({\Phi}^\top \Phi )^{-1} {\Phi}^\top (\Phi \theta + \epsilon_{1:n}) \\
&= \epsilon - \varphi(x)^\top ({\Phi}^\top \Phi )^{-1} {\Phi}^\top \epsilon_{1:n}
\end{align*}
So $r$ is a linear combination of independent Gaussian random variables $\epsilon_i$ with mean $\E[r] = 0$.
Denoting the empirical covariance as $\widehat{\Sigma} = \frac{1}{n} \sum_i^n \varphi(x_i) \varphi(x_i)^\top$ and the population covariance as $\Sigma = \E[\varphi(x_i) \varphi(x_i)^\top ]$, the variance of $r$ is
\begin{align*}
    \Var[r] &= \E[\epsilon^2] + \E \left[\varphi(x)^\top ({\Phi}^\top \Phi )^{-1} {\Phi}^\top \epsilon_{1:n} \epsilon_{1:n}^\top \Phi ({\Phi}^\top \Phi )^{-1} \varphi(x) \right] \\
    &= \sigma^2 + \E \left[ \varphi(x)^\top (\Phi^\top \Phi )^{-1} \Phi^\top \Phi (\Phi^\top \Phi )^{-1} \varphi(x)  \right] \sigma^2 \\
    &= \Big( 1 + \E \left[\varphi(x)^\top (\Phi^\top \Phi )^{-1} \varphi(x) \right] \Big)  \sigma^2 \\
    &= \left(1 + \frac{1}{n} \trace \Big[ \E [\Sigma \widehat{\Sigma}^{-1}] \Big] \right) \sigma^2 \\
    &= \left(1 + \frac{d}{n - d - 1} \right) \sigma^2
\end{align*}
The second equality is using that $\epsilon_i$ has variance $\sigma^2$. The last equality is using the fact that by considering independent unit Gaussian random variables $z_i = \Sigma^{-1 / 2} \varphi(x_i)$, so 
$\left(z_{1:n}^{\top} z_{1:n} \right)^{-1}$ follows Wishart distribution, and hence 
$\mathbb{E}\left[\operatorname{tr}\left(\Sigma \widehat{\Sigma}^{-1}\right)\right] = n \mathbb{E}\left[\operatorname{tr}\left(Z^{\top} Z\right)^{-1}\right] = \frac{n d}{n-d-1}$. The first part has been proved.

Next, we notice that $\varphi(x_i)^\top (\hat{\theta} - \theta) = \varphi(x_i)^\top ({\Phi}^\top \Phi )^{-1} {\Phi}^\top \epsilon_{1:n}$ is again a Gaussian random variable with mean $0$. Following similar analysis as above, the variance is $\frac{d}{n - d - 1} \sigma^2$. Therefore,
\begin{align*}
    \Pb \left( \left| |y_i - \varphi(x)^\top \hat{\theta}| - |y_i - \varphi(x)^\top \theta| \right| \leq t \right) & \geq  \Pb \left( \left| \varphi(x)^\top \hat{\theta} - \varphi(x)^\top \theta \right| \leq t \right) \\
    &= 1 - \frac{1}{\sqrt{\pi}} \Phi \left( \frac{t}{\sqrt{\frac{d}{n - d - 1}} \sigma \sqrt{\pi}} \right) \\
    &\asymp 1 - \frac{1}{\sqrt{\pi}} \Phi \left( t \sqrt{ \frac{ n}{\pi d \sigma^2 }} \right) \\
    &\asymp 1 - \frac{1}{\sqrt{\pi}} \exp \left( \frac{ - 2 n t^2 }{\pi d \sigma^2} \right) 
\end{align*}
By taking $t = \sigma \sqrt{\log n / n}$, we have $\Pb \left( \left| |y_i - \varphi(x)^\top \hat{\theta}| - |y_i - \varphi(x)^\top \theta| \right| \leq \sigma \sqrt{\frac{\log n}{n}} \right) \geq 1 - \frac{1}{\sqrt{\pi}} \left(\frac{1}{n}\right)^{\frac{2}{\pi d}} \geq 1 - \frac{1}{n}$.
\end{proof}

\begin{prop}[Perturb-one stability for OLS]\label{prop:perturb_one_stability}
Given samples $(x_1, y_1), \cdots, (x_n, y_n)$ with $y_i = \theta^\top \varphi(x_i) + \epsilon_i$ where $\epsilon_i$ are zero mean independent Gaussian random variables with variance $\sigma^2$, and another sample $(x_{n+m+1}, \overline{y})$. $x_{n+m+1}$ is not necessarily drawn from a same distribution as $x_1, \cdots, x_n$, and $\overline{y}$ is pre-selected from a bounded domain $\calY$. 
We have two OLS estimators, the first OLS estimator $\bar{\theta} = \left( \Phi^\top \Phi + \varphi(x_{n+m+1}) \varphi(x_{n+m+1})^\top \right)^{-1} \left( \Phi^\top y_{1:n} + \varphi(x_{n+m+1}) \overline{y} \right)$ is derived from using all the samples  and the second OLS estimator $\hat{\theta} = (\Phi^\top \Phi)^{-1} \Phi^\top y_{1:n}$ is derived from using all but the last sample.
Then with probability at leat $1 - \frac{1}{n}$, the predictive error under $\hat{\theta}$ and $\bar{\theta}$ are close to each other
\begin{align}
    \Pb \left( \left| \bar{r} - \hat{r} \right| \right) = 2 \exp \left(- \frac{n t^2}{d \sigma^2} \right)
\end{align}
where $\hat{r} = \epsilon_i - \varphi(x_i)^\top ({\Phi}^\top \Phi )^{-1} {\Phi}^\top \epsilon_{1:n}$ is the predictive error under $\hat{\theta}$ and $\bar{r} = \epsilon_i - \varphi(x_i)^\top ({\Phi}^\top \Phi + \varphi(x_{n+m+1}) \varphi(x_{n+m+1})^\top )^{-1} \left( \Phi^\top \epsilon_{1:n} + \varphi(x_{n+m+1}) \epsilon_{n+m+1} \right)$ is the predictive error under $\bar{\theta}$.
\end{prop}
\begin{proof}
Denote $\varphi(x_{n+m+1}) = \varphi$ and $\Gamma = {\Phi}^\top \Phi$.
\begin{align*}
\bar{r} - \hat{r} &= \varphi(x)^\top (\Gamma + \varphi \varphi^\top )^{-1} \left( \Phi^\top \epsilon_{1:n} + \varphi \epsilon_{n+m+1} \right) - \varphi(x)^\top \Gamma^{-1} {\Phi}^\top \epsilon_{1:n} \\ 
&= \varphi(x)^\top (\Gamma + \varphi \varphi^\top )^{-1} \Phi^\top \epsilon_{1:n} + \varphi(x)^\top (\Gamma + \varphi \varphi^\top )^{-1} \varphi \epsilon_{n+m+1} - \varphi(x)^\top \Gamma^{-1} {\Phi}^\top \epsilon_{1:n} \\
&= \varphi(x)^\top \left( (\Gamma + \varphi \varphi^\top )^{-1} - \Gamma^{-1} \right) \Phi^\top \epsilon_{1:n} + \varphi(x)^\top (\Gamma + \varphi \varphi^\top )^{-1} \varphi \epsilon_{n+m+1}
\end{align*}
Since $\epsilon_{1:n}$ are Gaussian random variables, and $\epsilon_{n+m+1}$ is a fixed constant, $\bar{r} - \hat{r}$ is also a Gaussian random variable whose mean $\mu$ and variance $V$ can be computed as follows.
\begin{align*}
    \mu &= \E[\varphi(x)]^\top (\Gamma + \varphi \varphi^\top )^{-1} \varphi \epsilon_{n+m+1} \\
    &= \epsilon_{n+m+1} \operatorname{tr} \left[ \E[\varphi(x)] \varphi^\top (\Gamma + \varphi \varphi^\top )^{-1} \right] \\
    &\leq \epsilon_{n+m+1} \operatorname{tr} \left[ \left( \sqrt{n} \E[\varphi(x)] \E[\varphi(x)]^\top + \frac{1}{\sqrt{n}} \varphi \varphi^\top  \right) \left( \Gamma + \varphi \varphi^\top \right)^{-1} \right] \\
    &\asymp \epsilon_{n+m+1} \operatorname{tr} \left[ \left( \sqrt{n} \E[\varphi(x)] \E[\varphi(x)]^\top  + \frac{1}{\sqrt{n}} \varphi \varphi^\top \right) \Big( n \E[\varphi(x)] \E[\varphi(x)]^\top  + \varphi \varphi^\top \Big)^{-1} \right] \\
    &= \frac{d}{\sqrt{n}}
\end{align*}
The third inequality is using $a a^\top + a b^\top \succeq 2 a b^\top$ and the fourth inequality is using concentration inequality for matrices \cite{tropp2012user} by noticing that $\frac{1}{n} \Gamma = \frac{1}{n} \sum_{i=1}^n \varphi(x_i) \varphi(x_i)^\top $ is the sample approximation of $\E[\varphi(x) \varphi(x)^\top] = \E[\varphi(x)] \E[\varphi(x)]^\top$.
\begin{align*}
    V &= \sigma^2 \varphi(x)^\top \left( (\Gamma + \varphi \varphi^\top )^{-1} - \Gamma^{-1} \right) \Phi^\top \Phi \left( (\Gamma + \varphi \varphi^\top )^{-1} - \Gamma^{-1} \right) \varphi(x) \\
    &= \sigma^2 \E\Big[ \varphi(x)^\top (\Gamma + \varphi \varphi^\top )^{-1} \varphi \varphi^\top \Gamma^{-1} \Gamma (\Gamma + \varphi \varphi^\top )^{-1} \varphi \varphi^\top \Gamma^{-1} \varphi(x) \Big] \\
    &\leq \sigma^2 \E\Big[ \varphi(x)^\top \Gamma^{-1} \varphi(x) \Big] \\
    &= \sigma^2 \operatorname{tr} \left[ \Gamma^{-1} \E[\varphi(x) \varphi(x)^\top] \right] \\
    &\asymp \frac{d}{n}\sigma^2
\end{align*}
The second equality is using $(\Gamma + \varphi \varphi^\top )^{-1} - \Gamma^{-1} = (\Gamma + \varphi \varphi^\top )^{-1} \varphi \varphi^\top \Gamma^{-1}$, the third inequality is using $(\Gamma + \varphi \varphi^\top )^{-1} \varphi \varphi^\top \prec I$ and the last inequality is using again concentration inequality for matrices \cite{tropp2012user}.

Now we have that
\begin{align*}
    \Pb \left( \left| \left| \bar{r} - \hat{r} \right| - \frac{d}{\sqrt{n}} \right| \leq t \right) & \geq 
    \Pb \left( \left|  \bar{r} - \hat{r} - \frac{d}{\sqrt{n}} \right| \leq t \right) \\
    &= 1 - \frac{1}{\sqrt{\pi}} \Phi \left( \frac{t}{ \sqrt{\frac{d}{n} \sigma^2} \sqrt{\pi}} \right) \\
    &= 1 - \frac{1}{\sqrt{\pi}} \Phi \left( \frac{t \sqrt{n}}{ \sqrt{\pi d \sigma^2}} \right) \\
    &\asymp 1 - \frac{1}{\sqrt{\pi}} \exp \left( \frac{ -2 n t^2 }{\pi d \sigma^2} \right)
\end{align*}
By taking $t = \sigma \sqrt {\frac{\log n }{n} }$, we have that $\Pb \left( \left| \bar{r} - \hat{r} \right| \leq \sigma \sqrt {\frac{\log n }{n} } \right) \geq 1 - \frac{1}{\sqrt{\pi}} \left(\frac{1}{n}\right)^{\frac{2}{\pi d}} \geq 1 - \frac{1}{n}$.
\end{proof}

\begin{prop}[Central limit theorem for weighted quantiles]\label{prop:weighted_quantile_clt}
Suppose $X_1, \cdots, X_n$ are i.i.d. continuous random variables from distribution with CDF $F_X$ and PDF $f_X$, and $w_1, \cdots, w_n$ are nonnegative weights that sum up to $1$. 
Denote effective sample size $n_{\text{eff}} = \sqrt{\sum_{i=1}^n w_i^2}$, under the assumption that there exist $\delta > 0$, such that $\lim_{n \to \infty} \frac{\sum_{i=1}^n w_i^{\delta+2}}{\left( \sum_{i=1}^n w_i^2 \right)^{\frac{2+\delta}{2}}} = 0$ then the $\beta$-th quantile of the weighted empirical distribution converge to a normal distribution as $n \to \infty$:
\begin{align*}
    \frac{1}{n_{\text{eff}}} \left( \operatorname{Quantile} \left( \beta; \sum_{i=1}^n w_i \delta_{X_i} \right) - F_X^{-1}(\beta) \right) \stackrel{d}{\longrightarrow} \calN \left(0, \frac{\beta (1-\beta)}{\left(f_X(F_X^{-1}(\beta))\right)^2}\right)
\end{align*}
\end{prop}

\begin{proof}
Let $Y_n(x)$ be a random variable defined for a fixed $x \in \mathbb{R}$ by weighted average $Y_n(x) = \sum_{i=1}^n w_i I\left\{X_i \leq x\right\}=\sum_{i=1}^n Z_i(x)$, where $Z_i(x)= w_i I\left\{X_i \leq x\right\}= w_i$ if $X \leq x$, and zero otherwise. 
Then $Z_i$ has expectation $\mu_i = w_i F_X(x)$ and variance $\sigma_i^2 = w_i^2 F_X(x)(1-F_X(x)) $. The assumption that $\lim_{n \to \infty} \frac{\sum_{i=1}^n w_i^{\delta+2}}{\left( \sum_{i=1}^n w_i^2 \right)^{\frac{2+\delta}{2}}} = 0$ ensures that Lyapunov’s condition is satisfied and so by the Lyapunov central limit theorem we have:
\begin{align*}
    \frac{1}{n_{\text{eff}} \sqrt{F_X(x)(1-F_X(x))} } \left( Y_n(x) - F_X(x) \right) \stackrel{d}{\longrightarrow} \calN \left(0, 1 \right)
\end{align*}

Now consider the transformation through function $g(t)$
defined for $0<t<1$ by $g(t)=F_X^{-1}(t)$. We have the first derivative of $g$ as
\begin{align*}
g^\prime(t)=\frac{d}{d t}\left( F_X^{-1}(t)\right) = \frac{1}{f_X\left(F_X^{-1}(t)\right)}
\end{align*}
Thus, using the delta method
\begin{align*}
    \frac{1}{n_{\text{eff}}} \left(F_X^{-1}\left(Y_n(x)\right)-F_X^{-1}\left(F_X(x)\right)\right) \stackrel{d}{\longrightarrow} \calN \left(0, \frac{F_X(x) \left( 1-F_X(x)\right) }{\left( f_X \left(F_X^{-1} \left(F_X(x) \right) \right) \right)^2}\right) 
\end{align*}
and writing $\beta=F_X(x)$, we have
\begin{align*}
    \frac{1}{n_{\text{eff}}} \left(F_X^{-1} \left(Y_n(x)\right) - x \right) \stackrel{d}{\longrightarrow} \calN \left(0, \frac{\beta (1-\beta)}{\left(f_X(x)\right)^2}\right)
\end{align*}
Note that $F_X^{-1}\left(Y_n(x)\right)$ is a random variable that equals the $\beta$-th quantile of the weighted empirical distribution $\sum_{i=1}^n w_i \delta_{X_i}$, and the proof is finished.
\end{proof}

%% file: icml_2024/appendix/4_lemma.tex
\section{Auxliary Lemmas}
\begin{lem}\label{lem:hoeffding_bernouli}
For $0 \leq p , \alpha \leq 1$, the following inequality holds 
\begin{align*}
    \sum\limits_{k= \lceil (1-\alpha)n \rceil}^n C_n^k p^k (1-p)^{n-k} \leq \exp\left(-2n (1-\alpha-p)^2 \right)
\end{align*}
\end{lem}
\begin{proof}
$X_1, \cdots, X_n$ are $n$ i.i.d Bernouli random variables with $\Pb(X_i = 1)=p$. 
Hoeffding inequality says that
\begin{align*}
    \Pb \left( \sum_{i=1}^n X_i-\mathbb{E}\left[\sum_{i=1}^n X_i\right] \geq t\right) \leq \exp \left(-\frac{2 t^2}{n} \right)
\end{align*}
Take $t=(1 - \alpha) n - np$, we have
\begin{align*}
    \Pb \left( \sum_{i=1}^n X_i \geq (1 - \alpha) n \right) \leq \exp\left(-2n (1-\alpha-p)^2 \right)
\end{align*}
Noticing that the left hand side is exactly $\sum\limits_{k= \lceil (1-\alpha)n \rceil}^n C_n^k p^k (1-p)^{n-k}$, so the lemma is proved.
\end{proof}

\begin{lem}[CDF for ordering statistic]\label{lem:order_statistic}
    For $n$ i.i.d random variables $X_1, \cdots, X_n$ whose cumulative distribution function is $F_X$, their order statistic $X_{(1)}, \cdots, X_{(n)}$ satsify $X_{(1)} \leq \cdots \leq X_{(n)}$. The cumulative distribution function for the $i$-th order statistic $X_{(i)}$ is
    \begin{align*}
        \Pb(X_{(i)} \leq x) =\sum_{k=i}^n  F_X(x)^k(1-F_X(x))^{n-k}
    \end{align*}
\end{lem}

\begin{lem}[Properties of half-normal distribution] \label{lem:half_normal} 
1. The $\alpha$-th quantile of $\big| \calN(0, \sigma^2) \big|$ is $\sqrt{2} \sigma \operatorname{erf}^{-1}(\alpha)$.
2. The cumulative distribution function of $\big| \calN(0, \sigma^2) \big|$ is $F(x) = \operatorname{erf}\left(\frac{x}{\sqrt{2} \sigma}\right)$.
3. The probability density function of $\big| \calN(0, \sigma^2) \big|$ is $f(x) = \sqrt{\frac{2}{\pi \sigma^2}} \exp\left( -\frac{x^2}{2\sigma^2} \right)$.
\end{lem}

\begin{lem}[Central Limit Theorem for Quantile]\label{lem:clt}
    $X_1, \cdots, X_n$ are n i.i.d sampled drawn from a distribution with cdf $F$ and pdf $f$, 
    then for a fixed $p \in(0,1)$, provided that the following conditions hold: $t \mapsto f\left(F^{-1}(t)\right)$ is continuous at the point $p$ and $f\left(F^{-1}(p)\right)>0$, we have that, as $n \rightarrow \infty$,
    \begin{align}
        \sqrt{n} X_{(n p)} \stackrel{d}{\rightarrow} \calN \left( \sqrt{n} F^{-1}(p), \frac{p(1-p)}{\left[f \left(F^{-1}(p)\right) \right]^2} \right),
    \end{align}
\end{lem}

\begin{lem}[Equation 7.1.13 of  \cite{abramowitz1948handbook}]\label{lem:erf}
    Denote $\Phi(x) = \int_x^\infty \exp^{-\frac{1}{2} t^2 } dt$,  then we have for $x \geq 0$:
    \begin{align}
        \frac{1}{x + \sqrt{x^2 + 1}} \exp(-2 x^2) \leq \Phi(x) \leq  \frac{1}{x + \sqrt{x^2 + 2 / \pi}} \exp(-2 x^2)
    \end{align}
    So when $|x|$ is very large, $\Phi(x) \asymp \exp(-2 x^2)$.
\end{lem}

%% file: icml_2024/appendix/3_experiments.tex
\section{Experiments}

\subsection{Experiments on Synthetic Data}

\label{subsec:app_exp_syn}

\noindent\textbf{Implementation Details.} 
For WCP, the propensity model is implemented as a logistic regression model, which is widely adopted in the causal inference literature. For density ratio estimation, we use the MLP model from scikit-learn\footnote{\url{https://scikit-learn.org/stable/}} to classify whether a given data point $(x,y)$ is from observational or interventional distribution.

\noindent\textbf{Results of Nested Methods for ITE.} 
We skipped the experiment for wTCP-DR as the nested methods from~\cite{lei2021conformal} for ITE requires inferring confidence intervals of potential outcomes on the massive $\mathcal{D}^O_{cal}$, leading to extremely heavy computational cost.
Table~\ref{tab:nested} shows results on ITE with nested inexact and exact methods which can construct ITE intervals from intervals of counterfactual outcomes.
As we can see, under the nested inexact method, none of the methods achieve 0.9 coverage, as this method does not guarantee coverage.
While the nested exact method can significantly expand the confidence interval, leading to low efficiency.

\noindent\textbf{Ablation Study on Density Estimation Method: MLP vs Density Estimator (DR).}
We compare two different density estimators, i.e., MLP from scikit-learn and density estimator densratio\footnote{\url{https://github.com/hoxo-m/densratio_py}} (DR) on the synthetic dataset, where we adopt the same setting as the results shown in Table~\ref{table:main_res_cf_outcome}.
Intuitively, directly modeling the density of the joint distribution (DR) is more challenging than classifying whether a data point is from the observational or the interventional distribution (MLP).
We can observe that the coverage of wTCP-DR drops significantly when DR is used, because an inaccurate estimate of density ratio would result in worse coverage of wTCP-DR. wSCP-DR (Exact and Inexact) are more robust against inaccurate density ratios due to the correction taken from the second-stage inference.

\noindent\textbf{Results with Different Settings.}
Here, we illustrate the results for different dimensionalities of the observed features ($dim(X)$) in Fig.~\ref{fig:impact_dim_x_interval} and results for different sample size of interventional data ($m$) in Fig.~\ref{fig:cevae_m_coverage}.
In Fig.~\ref{fig:impact_dim_x_interval}, we can observe that the coverage rates of all methoeds increase as $dim(X)$ grows, which corresponds to less hidden confounding.
At the same time, the interval widths of most of the methods become narrower when $dim(X)$ increases due to the decrease of calibration error of the underlying regression models given more informative observed features $X$.
For WCP, it only provides expected coverage guarantees when $dim(X)$ is large, which leads to weak hidden confounding and accurate estimates of propensity scores.
Its interval widths increase with $dim(X)$ such that the coverage can be guaranteed.
In Fig.~\ref{fig:cevae_m_coverage}, we show the coverage and interval width with $m$ ranging within $\{10,20,50,100,250,500,750,1,000\}$. For all methods, the coverage is increasing with $m$ and the interval width is decreasing with $m$, as expected. This is because, for small $m$, $m<50$, wTCP-DR cannot achieve the specified level of coverage (0.9) because the density ratio estimator has high variance. As $m$ increases, wTCP-DR reaches the coverage of 0.9 and the smallest interval width.

\begin{table}[ht]
\caption{Results of ITE on synthetic data under the nested inexact and exact methods~\cite{lei2021conformal}.}
\label{tab:nested}
\centering
\footnotesize
\begin{tabular}{lcccc}
\hline
Method & Coverage ITE (Nested Inexact) & Interval Width ITE (Nested Inexact) & Coverage ITE (Nested Exact) & Interval Width ITE (Nested Exact) \\ 
\hline
wSCP-DR(Inexact) & \(0.749 \pm 0.055\) & \(0.422 \pm 0.011\) & \(0.938	\pm 0.012\)	& \(0.767 \pm	0.011\)\\ 
wSCP-DR(Exact)   & \(0.819 \pm 0.033\) & \(0.504 \pm 0.009\) & \(0.948 \pm	0.016\) & 	\(0.847	\pm 0.008\) \\ 
WCP              & \(0.458 \pm 0.062\) & \(0.224 \pm 0.007\) & 
\(0.865 \pm	0.027\) &	\(0.602 \pm	0.006\)\\ 
Naive            & \(0.850 \pm 0.060\) & \(0.558 \pm 0.095\) & \(0.945	\pm 0.019\) &	\(0.943 \pm	0.104\)\\ 
\hline
\end{tabular}

\end{table}

\begin{table}[h]
\centering
\caption{Comparison of MLP and DR as density estimators with wTCP-DR and wSCP-DR (Inexact and Exact). The setting is the same as Table~\ref{table:main_res_cf_outcome}.}
\label{tab:MLP_vs_DR}
\footnotesize
\begin{tabular}{@{}lccccccc@{}}
\toprule
& Method &  Coverage $Y(0)$ \textuparrow
 & Interval Width $Y(0)$ \textdownarrow
 & Coverage $Y(1)$ \textuparrow
 & Interval Width $Y(1)$ \textdownarrow
 & Coverage ITE \textuparrow &  Interval Width ITE \textdownarrow
 \\ 
\midrule
MLP & wSCP-DR(Inexact) & 0.891 $\pm$ 0.026 & 0.414 $\pm$ 0.008 & 0.889 $\pm$ 0.019 & 0.421 $\pm$ 0.013 & 0.942 $\pm$ 0.017 & 0.835 $\pm$ 0.016\\
MLP & wSCP-DR(Exact) & 0.934 $\pm$ 0.026 & 0.496 $\pm$ 0.010 & 0.935 $\pm$ 0.023 & 0.503 $\pm$ 0.010 & 0.957 $\pm$ 0.018 & 0.998 $\pm$ 0.015 \\
MLP & wTCP-DR & 0.899 $\pm$ 0.028 & 0.386 $\pm$ 0.013 & 0.923 $\pm$ 0.015 & 0.576 $\pm$ 0.066 & 0.953 $\pm$ 0.015 & 0.962 $\pm$	0.074\\
DR & wSCP-DR(Inexact) & 0.899 $\pm$ 0.024 & 0.423 $\pm$ 0.013 & 0.874 $\pm$ 0.014 & 0.411 $\pm$ 0.011 & 0.946 $\pm$ 0.020 & 0.834 $\pm$ 0.015 \\
DR & wSCP-DR(Exact) & 0.936 $\pm$ 0.014 & 0.503 $\pm$ 0.009 & 0.934 $\pm$ 0.004 & 0.493 $\pm$ 0.017 & 0.966 $\pm$ 0.014 & 0.996 $\pm$ 0.009\\
DR & wTCP-DR & 0.847 $\pm$ 0.022 & 0.363 $\pm$ 0.011 & 0.853 $\pm$ 0.031 & 0.372 $\pm$ 0.013 & 0.910 $\pm$ 0.020 & 0.735 $\pm$ 0.016 \\
\bottomrule
\end{tabular}
\end{table}

\begin{figure*}[ht]
    \centering
    \subfloat[Coverage of $Y(0)$]{
\includegraphics[width=0.31\textwidth]{figs/cevae/x_dim/coverage_0_dr_use_Y_1_ite_method_naive.png}
        \label{fig:cevae_cover_conf_str_y0}
    }
    \hfill
    \subfloat[Coverage of $Y(1)$]{
\includegraphics[width=0.31\textwidth]{figs/cevae/x_dim/coverage_1_dr_use_Y_1_ite_method_naive.png}
        \label{fig:cevae_cover_conf_str_ite}
    }
    \hfill
    \subfloat[Coverage of ITE]{
\includegraphics[width=0.31\textwidth]{figs/cevae/x_dim/coverage_ITE_dr_use_Y_1_ite_method_naive.png}
        \label{fig:cevae_cover_conf_str_ite}
    } \hfill
    \subfloat[Interval width of $Y(0)$]{
\includegraphics[width=0.31\textwidth]{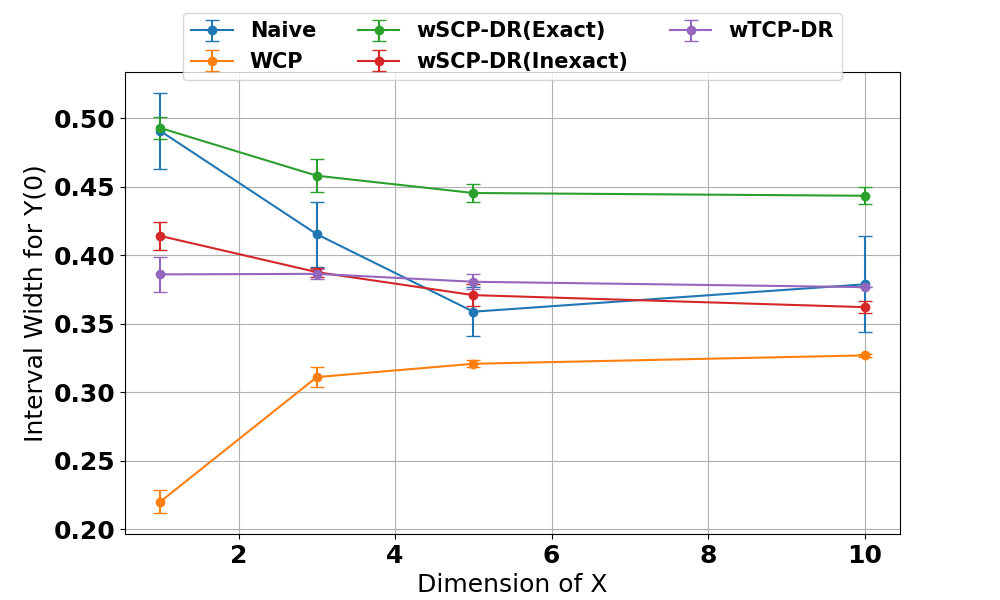}
        \label{fig:cevae_interval_conf_str_y0}
    }
    \hfill
    \subfloat[Interval width of $Y(1)$]{
\includegraphics[width=0.31\textwidth]{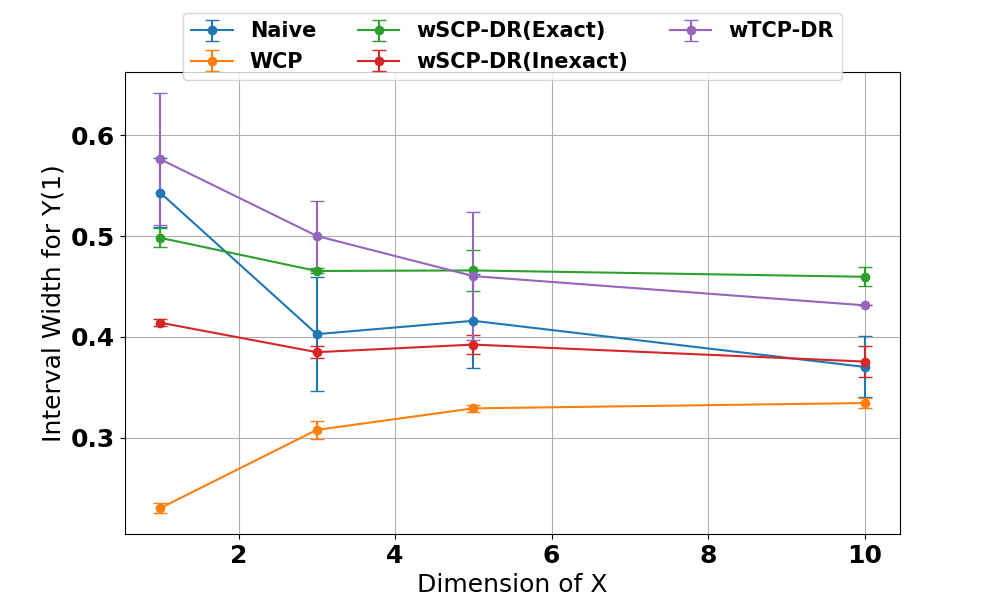}
        \label{fig:cevae_interval_conf_str_ite}
    }
    \hfill
    \subfloat[Interval width of ITE]{
\includegraphics[width=0.31\textwidth]{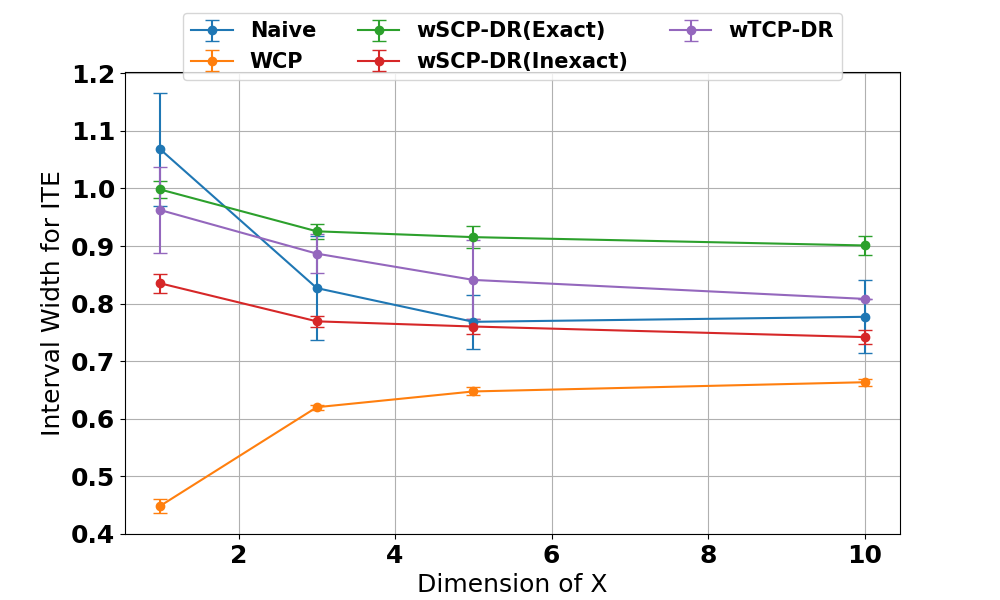}
        \label{fig:cevae_interval_conf_str_ite}
    }
    \caption{Coverage and interval width results of counterfactual outcomes and ITE with varying hidden confounding strength. Higher dimensional $X$ carries more information of the hidden confounders, leading to weaker hidden confounding.}
    \label{fig:impact_dim_x_interval}
\end{figure*}

\begin{figure}[htb!]
    \centering
    \subfloat[Coverage of $Y(0)$ with different $m$]{
\includegraphics[width=0.45\textwidth]{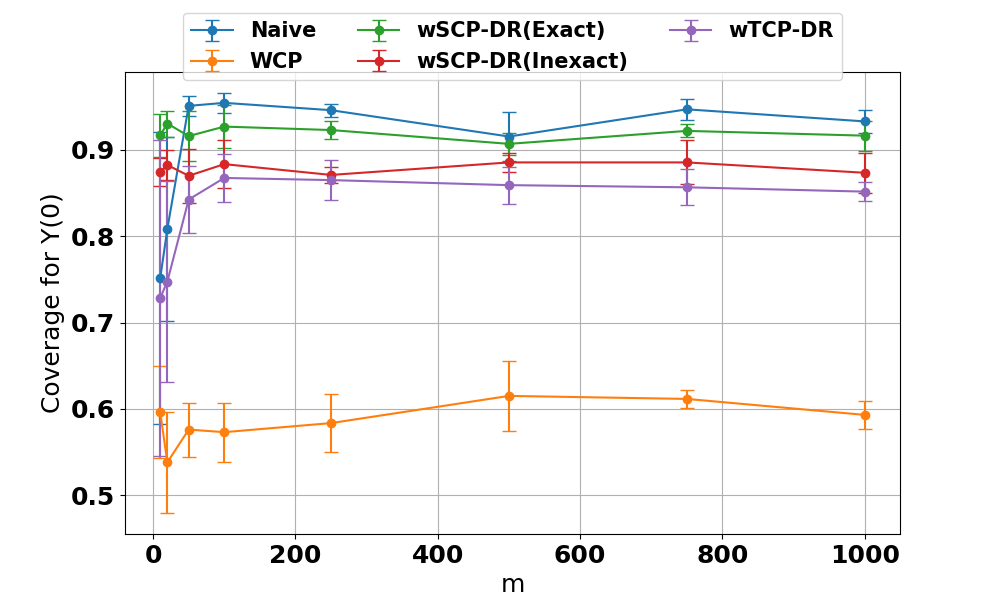}
        \label{fig:cevae_cover_y0_n_int}
    }
    \hfill
    \subfloat[Coverage of $Y(1)$ with different $m$]{
\includegraphics[width=0.45\textwidth]{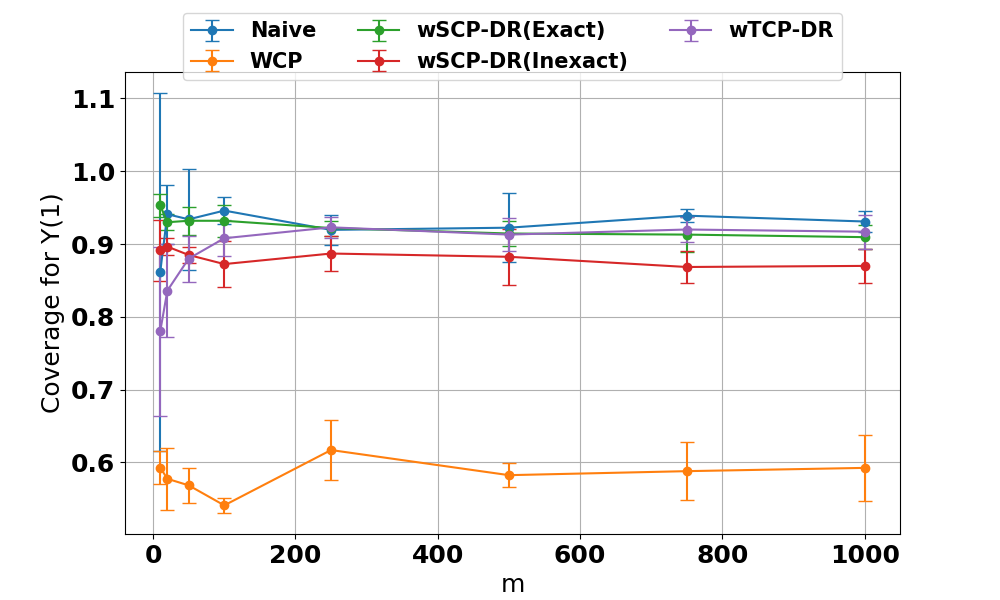}
        \label{fig:cevae_width_y1_n_int}
    }    \hfill
    \subfloat[Interval width of $Y(0)$ with different $m$]{
\includegraphics[width=0.45\textwidth]{figs/cevae/n_int/interval_width_Y0_x_dim_1.png}
        \label{fig:cevae_cover_y0_n_int}
    }
    \hfill
    \subfloat[Interval width of $Y(1)$ with different $m$]{
\includegraphics[width=0.45\textwidth]{figs/cevae/n_int/interval_width_Y1_x_dim_1.png}
        \label{fig:cevae_width_y1_n_int}
    }
    \caption{Impact of interventional data size $m$ on coverage and efficiency of conformal inference methods.}
    \label{fig:cevae_m_coverage}
    \vspace{10pt}
\end{figure}

\subsection{Experiments on Recommendation System Data}

\label{subsec:app_exp_rec}

\textbf{Implementation Details.} We use MSE loss to train matrix factorization (MF) models~\cite{koren2009matrix} with 64 dimensional embeddings as the base model for rating prediction, which is one of the most popular approaches in recommendation systems~\cite{schnabel2016recommendations,wang2018deconfounded}.
In this setting, the features (user/item embeddings) are learned from the factual outcomes $Y$, leading to their capability to capture part of hidden confounding.
We use the Python version of the package densratio for density ratio estimation of our method to handle the high dimensional.
For WCP-NB, following~\cite{schnabel2016recommendations,chen2021autodebias}, we fit a Naive Bayes classifier to model the propensity $P(T=1|X,Z,Y)$.
It is simplified as $P(T=1|Y)=\frac{P(Y|T=1)P(T=1)}{P(Y)}$. As $P(Y|T=0)$ is not available in the observational data, $P(Y)$ can only be estimated from the interventional data where treatment is randomized ($P(Y)=P^I(Y)=P^I(Y|T)$). So, WCP-NB needs to use interventional data with outcomes.
In this case, WCP-NB can be seens as a variant of our method using a different density ratio estimator based on propensity scores.

%
%



\begin{figure*}[ht]
    \centering
    \subfloat[Test empirical coverage with different $m_{cal}$ on Coat]{
        \includegraphics[width=0.45\textwidth]{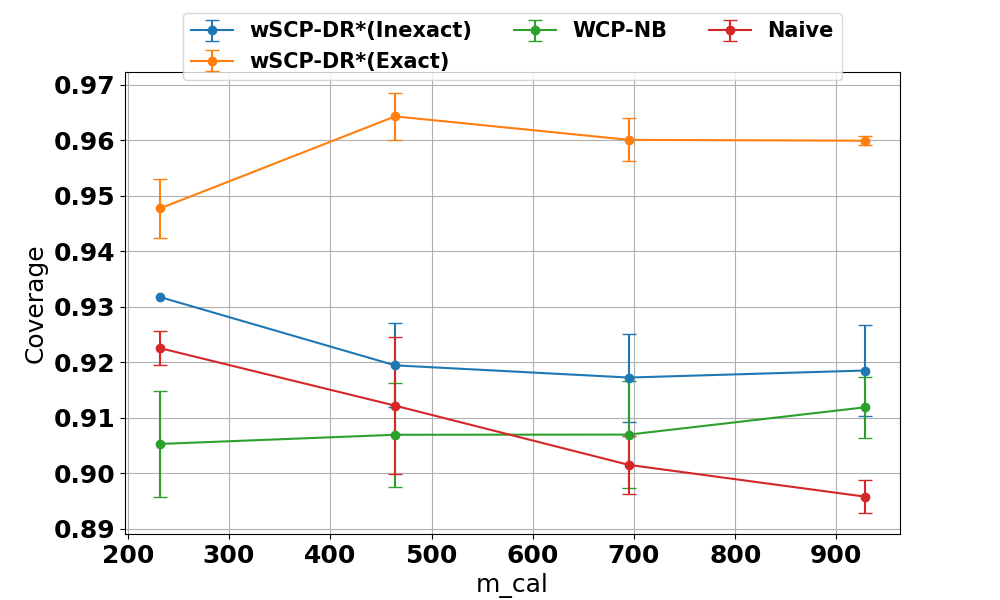}
        \label{fig:coat_m_cal_vs_cover}
    }
    \hfill
    \subfloat[Test interval width with different $m_{cal}$ on Coat]{
\includegraphics[width=0.45\textwidth]{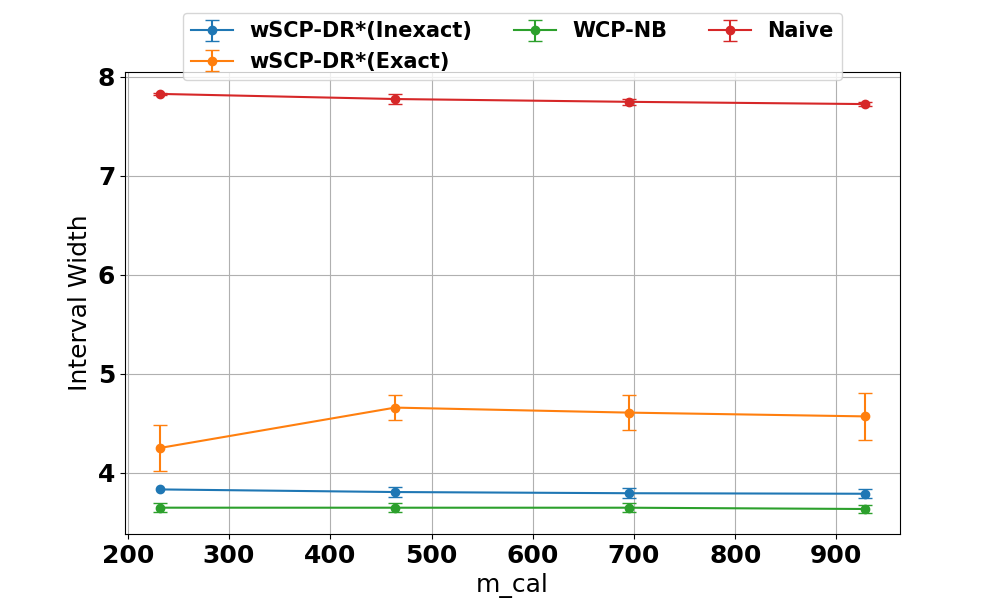}
        \label{fig:coat_m_cal_vs_width}
    }
    \caption{Results on Coat with different $m_{cal}$}
    \label{fig:m_cal}
\end{figure*}

\noindent\textbf{Impact of $m_{cal}$.} We maintain $m_{tr}=0.2m, m_{ts}=0.6m$ and modify $m_{cal}\in\{0.05m, 0.1m, 0.15m, 0.2m\}$. Results are shown in Fig.~\ref{fig:m_cal}. All the methods maintain coverage close or above 0.9 for all cases.
In terms of efficiency, we can observe that the efficiency of Naive gets slightly improved with increasing $m_{cal}$.

\begin{figure}[ht]
    \centering
    \subfloat[Test empirical coverage with different $m_{tr}$ on Coat]{
        \includegraphics[width=0.45\linewidth]{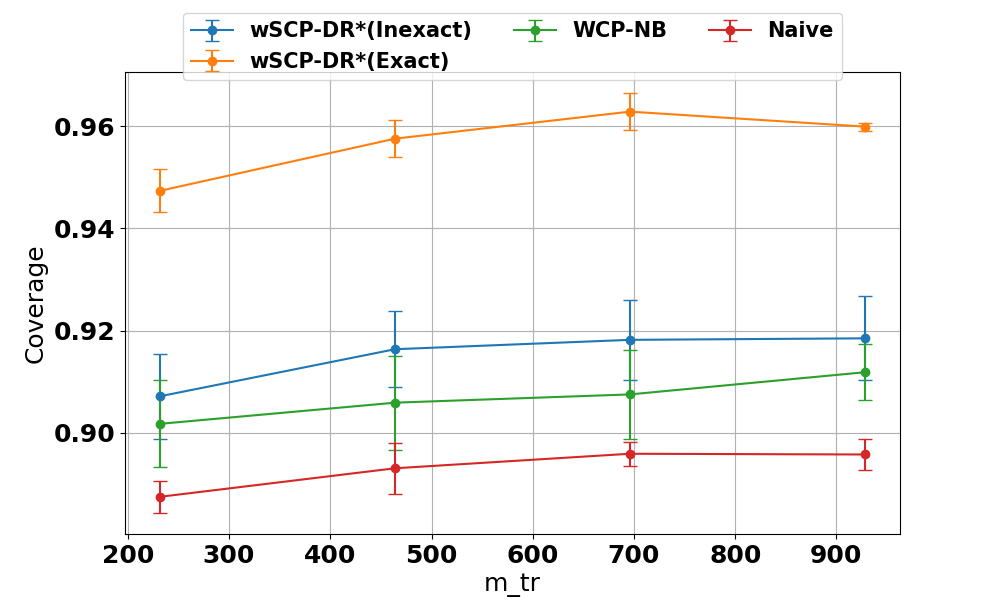}
        \label{fig:coat_m_tr_vs_cover}
    }
    \hfill
    \subfloat[Test interval width with different $m_{tr}$ on Coat]{
\includegraphics[width=0.45\linewidth]{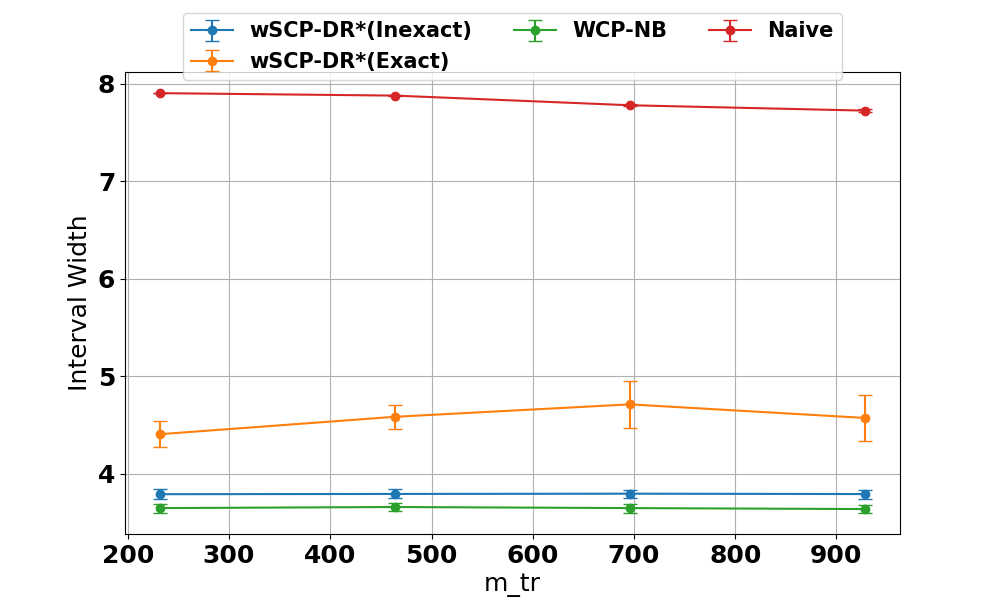}
        \label{fig:coat_m_tr_vs_width}
    }
    \caption{Results on Coat with different $m_{tr}$}
    \label{fig:coat_m_tr}
\end{figure}

\noindent\textbf{Impact of $m_{tr}$.} We maintain $m_{cal}=0.2m, m_{ts}=0.6m$ and modify $m_{tr}\in\{0.05m, 0.1m, 0.15m, 0.2m\}$.
Fig.~\ref{fig:coat_m_tr} shows results on Coat where $m$ is small.
We make the following observations.
First, the efficiency of Naive is improved because its base model has lower MSE with more training data, leading to smaller confidence intervals.
Second, the coverage of all methods are improved, as more trainig samples from the interventional distribution can improve the base model for the Naive method, density ratio estimators for our methods and the propensity model for WCP-NB.